\documentclass{article}

\usepackage{microtype}
\usepackage{graphicx}
\usepackage{subfigure}
\usepackage{booktabs} 

\usepackage{physics}
\usepackage{amsmath,amsthm,amssymb}
\usepackage{mathtools}
\usepackage{bm}

\usepackage{hyperref}

\newcommand{\W}{{\mathcal{W}}}
\newcommand{\WU}[1]{\W^{#1}}
\newcommand{\WL}[1]{\W_{#1}}
\newcommand{\WUL}[2]{{\WU{#1}_{#2}}}
\newcommand{\WLU}[2]{{\WL{#1}^{#2}}}

\newcommand{\WB}{{\mathcal{B}}}

\newcommand{\WBL}[1]{\WB_{#1}}

\newcommand{\VWB}{{\mathcal{VB}}}

\newcommand{\VWBL}[1]{\VWB_{#1}}

\newcommand{\VW}{{\widetilde{\mathcal{W}}}}

\newcommand{\VWL}[1]{\VW_{#1}}

\newcommand{\VWLU}[2]{{\VWL{#1}^{#2}}}

\newcommand{\Dist}{{c}}

\newcommand{\h}{{h}}

\newcommand{\hL}[1]{\h_{#1}}

\newcommand{\hh}{{\bm{h}}}

\newcommand{\hhL}[1]{\hh_{#1}}

\newcommand{\PI}{{\pi}}

\newcommand{\PIL}[1]{\PI_{#1}}

\newcommand{\PS}{{\Pi}}

\newcommand{\PSL}[1]{\PS_{#1}}

\newcommand{\Region}{{\mathcal{R}}}

\newcommand{\RegionL}[1]{\Region_{#1}}

\newcommand{\RegionLU}[2]{\RegionL{#1}^{#2}}

\newcommand{\RRegion}{{\bm{\mathcal{R}}}}

\newcommand{\RRegionL}[1]{\RRegion_{#1}}

\newcommand{\x}{{x}}

\newcommand{\xL}[1]{\x_{#1}}

\newcommand{\X}{{\mathcal{X}}}

\newcommand{\XL}[1]{\X_{#1}}

\newcommand{\y}{{y}}

\newcommand{\yL}[1]{\y_{#1}}
  
\newcommand{\yLU}[2]{\yL{#1}^{#2}}  

\newcommand{\Y}{\mathcal{Y}}

\newcommand{\T}{{T}}
\newcommand{\TU}[1]{\T^{#1}}
\newcommand{\TL}[1]{\T_{#1}}

\newcommand{\TLU}[2]{\TL{#1}^{#2}}

\newcommand{\MUS}{{\mathcal{P}}}

\newcommand{\LagPhi}{{\varphi}}

\newcommand{\LagPhiL}[1]{\LagPhi_{#1}}

\newcommand{\Weight}{{\lambda}}

\newcommand{\WeightL}[1]{\Weight_{#1}}

\newcommand{\Energy}{{I}}

\newcommand{\EnergyL}[1]{\Energy_{#1}}

\newcommand{\MU}{{\mu}}

\newcommand{\MUL}[1]{\MU_{#1}}

\newcommand{\MUNL}[1]{{\bm{\MU}_{#1}}}

\newcommand{\NU}{{\nu}}
\newcommand{\NUU}[1]{\NU^{#1}}
\newcommand{\NUL}[1]{\NU_{#1}}
\newcommand{\NUUL}[2]{\NUU{#1}_{#2}}
\newcommand{\NULU}[2]{\NUL{#1}^{#2}}

\newcommand{\IdxCentroid}{{k}}
\newcommand{\IdxCentroidSecond}{{\ell}}

\newcommand{\K}{{K}}
\newcommand{\N}{{N}}

\newcommand\eqdef{\stackrel{\mathclap{\normalfont\mbox{\normalfont\tiny def}}}{=}}
\newcommand{\ignore}[1]{}

\DeclareMathOperator*{\argmin}{arg\,min}

\newtheorem{proposition}{Proposition}

\newtheorem{corollary}{Corollary}

\graphicspath{{figs/}}

\usepackage[review]{icml2020}

\icmltitlerunning{Variational Wasserstein Barycenters}

\begin{document}

\twocolumn[
\icmltitle{\ignore{Geometric Clustering via Variational Wasserstein Barycenters\\ . \\}
Variational Wasserstein Barycenters for Geometric Clustering}



\icmlsetsymbol{equal}{*}

\begin{icmlauthorlist}
\icmlauthor{Liang Mi}{asu}
\end{icmlauthorlist}

\icmlaffiliation{asu}{Arizona State University, USA}

\icmlcorrespondingauthor{Liang Mi}{liangmi@asu.edu}

\icmlkeywords{Machine Learning, ICML, Optimal Transport, Wasserstein Barycenter, Clustering, K-means, Variational Method}

\vskip 0.3in
]



\printAffiliationsAndNotice{}  

\begin{abstract}
We propose to compute Wasserstein barycenters (WBs) by solving for Monge maps with variational principle. We discuss the metric properties of WBs and explore their connections, especially the connections of Monge WBs, to K-means clustering and co-clustering. We also discuss the feasibility of Monge WBs on unbalanced measures and spherical domains. We propose two new problems -- regularized K-means and Wasserstein barycenter compression. We demonstrate the use of VWBs in solving these clustering-related problems.
\end{abstract}

\section{Introduction}
\label{sec:intro}
Clustering distributional data according to their spatial similarities has been a core issue in machine learning. \ignore{K-means~\cite{lloyd1982least,forgy1965cluster} and Gaussian mixture modeling (GMM)~\cite{} are two of the most widely adopted formulations for a clustering problem.}Numerous theories and algorithms for clustering problems have been developed to help understand the structure of the data and to discover homogeneous groups\ignore{ and subgroups} in their embedding spaces\ignore{, usually the Euclidean space}. Clustering algorithms also apply to unsupervised learning problems that pass information from known centroids to unknown empirical samples\ignore{~\cite{}}. Occasionally, researchers regard clustering as finding the optimal semi-discrete correspondence between distributional data or vice versa.\ignore{, e.g.~\cite{mi2018regularized}.}

Optimal transportation (OT) techniques have gained increasing popularity in the past two decades for measuring the distance between distributional data as well as aligning them together. Rooted in the OT theories, several OT-based clustering algorithms have emerged in recent years as alternatives, thanks to their efficiency and robustness. In these works, the researchers discovered the connections between different clustering problems and the OT problem through the Wasserstein barycenter (WB) formulation which computes a ``mean'' of one or multiple distributions. However, most of them deliver the results as \textit{soft assignments} that need to be further discretized.

In this paper, we propose to compute the Wasserstein barycenter based on Monge OT and explore its natural connections to different clustering problems that prefer \textit{hard assignments}. We base our OT solver on variational principles and coin our method as variational Wasserstein barycenters. We study the metric properties of WBs and use them to explain and solve different clustering-related problems such as regularized K-means clustering, co-clustering, and vector quantization and compression. We also show its immunity to unbalanced measures and its extension to measures on spherical domains. We discuss our method from different angles through comparison with other barycenter methods. We show the advantages of Monge OT-based barycenters in solving geometric clustering problems. We are among the first few that compute Monge barycenters and discover its connections to clustering problems.

\section{Related Work and Our Contributions} \label{sec:related}

\ignore{Clustering analysis dates back to 1932 by *** for studying cultural relationships.}Computational clustering algorithms date back to~\cite{lloyd1982least,forgy1965cluster} for solving K-means problems. From then, researchers have proposed different formulations and algorithms such as spectral clustering\ignore{~\cite{ng2002spectral}} and density-based clustering\ignore{~\cite{ester1996density}}. Mixture modeling, especially Gaussian mixture modeling, is also considered to be a robust solution to clustering problems\ignore{~\cite{mclachlan1988mixture}}. Hierarchical clustering\ignore{~\cite{}} and co-clustering\ignore{~\cite{}} also attracted much attention in the machine learning community. \cite{xu2005survey} surveys some classic clustering algorithms. \ignore{Some typical work include ****.} The term ``\textit{geometric clustering}'' appeared in the early literature, such as \cite{murtagh1983survey,quigley2000fade}, referring to clustering samples into subspaces according to their location in the metric space, usually the Euclidean space. In~\cite{applegate2011unsupervised}, the authors discuss the connection between K-means and another famous problem -- the OT distance, or the Wasserstein distance. 

The transportation problem has attracted many mathematicians since its very birth. Monge first raised the problem~\cite{monge1781memoire} as finding a measure-preserving map between probability measures; Kantorovich extended the problem to finding a joint probability measure~\cite{kantorovich1942translocation}; Brenier further connected the OT problem to fluid dynamics and convex geometry~\cite{brenier1991polar}. It's early applications include comparing 1D histograms for image retrieval~\cite{rubner2000earth}. Thanks to efficient OT solvers, e.g., \cite{cuturi2013sinkhorn}, \ignore{that scale to millions of samples in high-dimensional spaces, }OT has become a popular tool in machine learning with which we compare distributional data.

Meanwhile, by regarding the OT distance as a metric, we can interpolate in the space of probability measure. \cite{mccann1997convexity} laid the foundation; \cite{agueh2011barycenters} developed the problem into a general scenario and coined the term ``\textit{Wasserstein barycenters}''. \cite{cuturi2014fast,ho2017multilevel,mi2018variational} relate WBs to K-means like clustering problems and \cite{leclaire2019fast,lee2019hierarchical} explored the use of OT for hierarchical clustering. \cite{claici2018stochastic} is among the latest work on scalable semi-discrete Wasserstein barycenters. Most of them follow Kantorovich's static OT; few of them follow Monge's, or Brenier's, dynamic version that regards OT as a gradient flow in the probability space.

Compared to previous work, our contribution is three-fold: \ignore{1) we provide an alternative proofs for connecting Brenier's OT to Kantorovich's OT; }1) We derive the WB based on Monge's OT formulation and explore its connections to different clustering problems; 2) We prove the metric properties of our WB and propose it as a metric for evaluating multi-marginal clustering algorithms; 3) We explore the advantages and disadvantages of Monge WB through empirical comparison with other methods.

\section{Primer on Optimal Transportation}
\label{sec:primer}
We begin by iterating key concepts of optimal transportation (OT), variational OT, and Wasserstein barycenters (WBs). Suppose $\MU, \NU$ are \textit{Borel probability distributions} supported in \textit{Polish spaces} $\X(\x)$, $\Y(\y)$, respectively. Let $\MUS(\X \times \Y)$ be the set of all probability distributions on $\X \times \Y$.
Then, we denote by $\PS(\MU,\NU) = \{\PI \in \MUS(\X \times \Y)\ |\ \int_{\X}d\PI(\x, \y) = d\NU(\y), \int_{\Y}d\PI(\x, \y) = d\MU(\x) \}$ the set of all transportation maps $\PI$ between $\MU$ and $\NU$. 
Thus, $\PI$ is also the joint distribution of $\MU$ and $\NU$ and $d\PI(\x, \y)$ specifies the \textit{mass} transported across $\x$ and $\y$. 
In addition, we use $\Dist(\x, \y): \X \times \Y \rightarrow \mathbb{R}^{\geq 0}$ to specify the transportation cost between $\x$ and $\y$.

\subsection{Optimal Transportation}
\label{sec:ot}
The OT problem is to minimize the total transportation cost:
%
\begin{equation}
    \min_{\PI \in \PS(\MU,\NU)} \EnergyL{1}[\PI] = \int_{\X \times \Y} \Dist(\x, \y)^p d\PI(\x, \y), \nonumber
\end{equation}
%
where $p \in [1, \infty)$ indicates the moment of the cost function. Then, we call this minimum the \textit{p-Wasserstein distance}: 
%
\begin{equation}
    \WL{p} = \underset{\PI \in \PS(\MU,\NU)}{\inf} \left(\EnergyL{1}[\PI]\right)^{1/p}. \nonumber
\end{equation}
%
The above is the well-known Kantorovich's OT formulation that allows a partial map that splits the measure $d\MU(\x)$ during transportation. In Monge's original version, each location $\x$ has a unique correspondence $\y$. If we define such a map as $\T: \X \rightarrow \Y$, then we have $d\PIL{\T}(\x, \y) \equiv d\MU(x)\delta[\y = \T(\x)]$ and Monge OT:
%
\begin{equation} \label{eq:monge}
   \TU{*} = \argmin_{\PI_{\T} \in \PS(\MU,\NU)} \EnergyL{1}[\PI_{\T}] \equiv \int_{\X} \Dist(\x, \T(\x))^p d\MU(\x)
\end{equation}
%
$\T$ \textit{pushes forward} $\MU$ to $\NU$, i.e. $\NU = \T\#\MU$; more rigorously, for any measurable set $B \subset \Y,\ \NU[B] = \MU[\TU{-1}(B)]$. We direct readers to \cite{villani2003topics, peyre2019computational} for more on OT. In this paper, we focus on Monge OT. In particular, we narrow our discussion to $\X, \Y \subseteq \mathbb{R}^{n}$, $c(\x, \y)= \|\x - \y\|_{2}$, and $p = 2$ unless specified otherwise. Hence, we compute $\WL{2}$.

\subsection{Variational Optimal Transportation}
\label{sec:vot}
Directly computing a Monge map is highly intractable and variational methods have been adopted by many researchers. \cite{de2012blue, gu2013variational,levy2015numerical} offer three variational formulations. We follow~\cite{gu2013variational} and in this paper refer to it as \textit{variational OT} or VOT.

Suppose $\NU$ is supported on $\K$ discrete atoms $\bm{\y} = \{\yL{\IdxCentroid}\}_{\IdxCentroid=1}^{\K} \subset \Y$. The problem becomes \textit{semi-discrete OT}. VOT starts with a piece-wise linear function $\theta_{\bm{\h}}(\x) = \underset{\IdxCentroid}{\max} \{\x \yL{\IdxCentroid} + \hL{\IdxCentroid}\}$. Each $\yL{\IdxCentroid}$ associates with a \textit{height} $\hL{\IdxCentroid}$. The gradient, $\nabla\theta_{\bm{\h}}(x) = \yL{\IdxCentroid}$ where $\IdxCentroid$ induces the maximum, serves as a map from $\X$ to $\Y$. It induces a graph: $ \bm{\Region}_{\bm{\h}} = \bigcup\limits_{\IdxCentroid=1}^{\K} \left(\RegionL{\bm{\h}}\right)_{\IdxCentroid}, \left(\RegionL{\bm{\h}}\right)_{\IdxCentroid} \eqdef \{\x \in \X\ |\ \x\yL{\IdxCentroid} + \hL{\IdxCentroid} \geq \x\yL{\IdxCentroidSecond} + \hL{\IdxCentroidSecond}, \forall \IdxCentroidSecond \neq \IdxCentroid\}$. For simplicity, we remove $\bm{\h}$ and use $\RegionL{\IdxCentroid}$ instead. We introduce an energy:
\begin{equation}\label{eq:vot}
        \ignore{\min_{\bm{\h}}}\ \EnergyL{2}[\bm{\h}] \eqdef \int_{\bm{0}}^{\bm{\h}} \sum_{\IdxCentroid=1}^{\K} \int_{\RegionL{\IdxCentroid}} d\MU(\x) d\hL{\IdxCentroid} - \sum_{\IdxCentroid = 1}^{\K}\NU(\yL{\IdxCentroid})\hL{\IdxCentroid},
\end{equation}
whose gradient, $\big\{\int_{\RegionL{\IdxCentroid}} d\MU(\x) - \NU(\yL{\IdxCentroid})\big\}_{\IdxCentroid}$, also integrates to 
\begin{equation}\label{eq:theta}
    \ignore{\min_{\bm{\h}}}\ \EnergyL{3}[\bm{\h}] \eqdef \int_{\X} \theta_{\bm{\h}}(\x) d\MU(\x) - \sum_{\IdxCentroid = 1}^{\K}\NU(\yL{\IdxCentroid})\hL{\IdxCentroid}.
\end{equation}
Meanwhile, the \textit{Lagrangian duality} of Monge OT \eqref{eq:monge} is 
\begin{equation} \label{eq:monge_dual}
\begin{gathered}
    \max_{\bm{\LagPhi}}\ \min_{\T}\ \EnergyL{4}[\bm{\LagPhi},T] 
    \eqdef \\ \int_{\X} \big( \|\x - \T(\x)\|_{2}^{2} + \sum_{\IdxCentroid=1}^{\K} \LagPhiL{\IdxCentroid} \big) d\MU(x)  - \sum_{\IdxCentroid = 1}^{\K} \LagPhiL{\IdxCentroid}\NU(\yL{\IdxCentroid}),
\end{gathered}
\end{equation}
where $\bm{\LagPhi} = \{\varphi_\IdxCentroid\}_{\IdxCentroid=1}^{\K}$. \eqref{eq:monge_dual} simplifies to 
\begin{equation}\label{eq:monge_ot_dual}
\begin{gathered}
    \max_{\bm{\LagPhi}}\ \EnergyL{4}[\bm{\LagPhi}] \\
    = \sum_{\IdxCentroid = 1}^{\K} \int_{\RegionL{\IdxCentroid}'} \big(\|\x - \yL{\IdxCentroid}\|_{2}^{2} + \LagPhiL{\IdxCentroid} \big) d\MU(\x) - \sum_{\IdxCentroid = 1}^{\K} \LagPhiL{\IdxCentroid}\NU(\yL{\IdxCentroid}),
\end{gathered}
\end{equation}
$\RegionL{\IdxCentroid}' = \{\x \in \X\ |\ \|\x - \yL{\IdxCentroid}\|_{2}^{2} + \LagPhiL{\IdxCentroid} \leq \|\x - \yL{\IdxCentroidSecond}\|_{2}^{2} + \LagPhiL{\IdxCentroidSecond}, \forall \IdxCentroidSecond \neq \IdxCentroid\}$ which coincides with a \textit{power Voronoi diagram}. 

We provide detailed derivation for above formulas in Appendix and then prove their following connections.
\begin{proposition} \label{th:energy_connection}
    \textbf{1}. The minimum point of $\EnergyL{2}[\bm{\h}]$, \eqref{eq:vot}, also minimizes $\EnergyL{3}[\bm{\h}]$, \eqref{eq:theta}. \textbf{2}. $\RegionL{\IdxCentroid} \equiv \RegionL{\IdxCentroid}'$. \textbf{3}. $\bm{\Region}$ in $\EnergyL{2}[\bm{\h}]$, \eqref{eq:vot}, induces the Monge map $\T: x \rightarrow \yL{\IdxCentroid}$. \textbf{4}. Minimizing $\EnergyL{2}[\bm{\h}]$, \eqref{eq:vot}, is equivalent to maximizing $\EnergyL{4}[\bm{\h}]$, \eqref{eq:monge_ot_dual}. 
\end{proposition}
Therefore, we ``variationally'' minimize $\EnergyL{2}[\bm{\h}]$, \eqref{eq:vot}, for a \textit{height vector} $\bm{\h}$ and that will produce a Monge map $\TU{*}$.

\subsection{Wasserstein Barycenters}
\label{sec:wb}
The Wasserstein distance (WD) satisfies all metric properties. The \textit{fr\'{e}chet mean} of a collection of distributions $\MUNL{1:N} \eqdef \{\MUL{i}\}_{i=1}^{N}$ w.r.t the WD is called the \textit{Wasserstein barycenter} (WB). It is the minimizer of the weighted average: 
%
\begin{equation} \label{eq:wb}
    \NU = \underset{\NU \in \MUS(\Y)}{\arg\min} \sum_{i = 1}^{\N} \lambda_{i} \WUL{2}{2}(\MUL{i}, \NU),
\end{equation}
%
for $\WeightL{i} \in [0, 1]$ and $\sum_i \WeightL{i} = 1$. We simplify \eqref{eq:wb} by assuming uniform weights and rewrite it as
%
\begin{equation} \label{eq:wb2}
    \NU = \underset{\NU \in \MUS(\Y)}{\arg\min} \frac{1}{\N} \sum_{i = 1}^{\N} \int_{\XL{i}}\|\x - \TLU{i}{*}(\x)\|_{2}^{2}d\MUL{i}(\x),
\end{equation}
%
$s.t.\ \TLU{i}{*}\#\MUL{i} = \NU$ OT for all $i$. Suppose the barycenter $\NU$ is supported on $\K$ discrete atoms $\bm{\y} = \{\yL{\IdxCentroid}\}_{\IdxCentroid=1}^{\K}$. If we fix $\NU(\yL{\IdxCentroid})$ and only allow updating $\bm{\y}$, then readers can notice that \eqref{eq:wb2} is simultaneously solving $\N$ \textit{constrained K-means} problems\ignore{~\cite{cuturi2014fast}} using the same set of centroids with fixed capacity, $\bm{\NU} = \{\NU(\yL{\IdxCentroid})\}_{\IdxCentroid=1}^{\K}$. $\TLU{i}{*}$ serves as the optimal \textit{assignment function} in each K-means problem. Note that $\TLU{i}{*}(\x)$ is a \textit{hard assignment} that has only one target because we solve Monge OT. 

To clarify notation, we use $\NU$ to denote the probability distribution whether continuous or discrete. If it is discrete, namely a collection of Dirac measures, then we use $\bm{y}$ and $\bm{\NU}$ to denote its supports and measures. $\yL{\IdxCentroid}$ and $\NUL{\yL{\IdxCentroid}}$ specify the location and measure of the $\IdxCentroid$th Dirac measure.

\section{Variational Wasserstein Barycenters}
\label{sec:vwb}
Solving the WB problem relies on alternatively solving $\N$ OT problems and updating the barycenter, $\NU$. Eventually, $\NU$ minimizes the average WD between empirical distributions and the barycenter. A discrete distribution $\NU$ consists of support and measure $(\bm{\y}, \bm{\NU}) = \{(\yL{\IdxCentroid}, \NUL{\IdxCentroid})\}_{\IdxCentroid=1}^{\K}$. Updating both of them, e.g., \cite{ye2017fast}, is difficult and even troublesome in some cases (see Appendix). In this paper, we follow~\cite{cuturi2014fast} and only update one of them while fixing the other throughout the optimization.

\subsection{Discrete Barycenters via VOT}
\label{sec:wbvot}
We first solve $N$ VOT problems~\eqref{eq:vot}:
%
\begin{equation}\label{eq:vwb}
\begin{gathered}
    \min_{\{\bm{\hL{i}}\}_{i=1}^{\N}} \EnergyL{5}[\{\bm{\hL{i}}\}] \\
    \eqdef \frac{1}{\N} \sum_{i=1}^{\N} \left( \int_{\bm{0}}^{\bm{\hL{i}}} \sum_{\IdxCentroid=1}^{\K} \int_{\RegionL{i, \IdxCentroid}} d\MUL{i}(x) d\hL{i, \IdxCentroid} - \sum_{\IdxCentroid = 1}^{\K}\NUL{\IdxCentroid}\hL{i, \IdxCentroid} \right)\\ \nonumber
\end{gathered}
\end{equation}
%
Its derivative w.r.t. the VOT optimizer $\hL{i, \IdxCentroid}$ is
%
\begin{equation}\label{eq:vwb_dh}
    \nabla \EnergyL{5}[\bm{\hL{i}}] = \left\{\pdv{\EnergyL{5}}{\hL{i, \IdxCentroid}} = \int_{\RegionL{i, \IdxCentroid}} d\MUL{i}(\x) - \NUL{\IdxCentroid}\right\}_{\IdxCentroid = 1}^{\K},
\end{equation}
%
which, in practice, can be replaced by its stochastic version,
%
\begin{equation}\label{eq:vwb_sdh}
    \pdv{\EnergyL{5}}{\hL{i, \IdxCentroid}} \approx \sum_{\x \in \RegionL{i, \IdxCentroid}} \MUL{i}(\x) - \NUL{\IdxCentroid}, \nonumber
\end{equation}
%
where $x$'s are now Monte Carlo samples. Then, we can naturally adopt the gradient descent (GD) update:
%
\begin{equation}\label{eq:vwb_gd}
    \bm{\hL{i}}^{(t+1)} = \bm{\hL{i}}^{(t)} - \eta \nabla \EnergyL{5}[\bm{\hL{i}}].
\end{equation}

For completeness, we give the second-order derivative in Appendix. Its computation, however, involves integrating over the Voronoi facets and thus is intractable in general. \ignore{An exception is when $\MUL{i}$'s are all uniform distributions.}

To solve for $\NU$, we rewrite the objective of the WB~\eqref{eq:wb2} as
%
\begin{equation} \label{eq:wb3}
\begin{split}
    \min_{\NU \in \MUS(\Y)} \EnergyL{6}[\NU]
    &\eqdef \frac{1}{\N} \sum_{i = 1}^{\N} \int_{\XL{i}}\|\x - \TLU{i}{*}(\x)\|_{2}^{2}d\MUL{i}(\x) \\
    & = \frac{1}{\N} \sum_{i = 1}^{\N} \sum_{\IdxCentroid = 1}^{\K} \int_{\RegionL{i, \IdxCentroid}}\|\x - \yL{\IdxCentroid}\|_{2}^{2}d\MUL{i}(\x),
\end{split}
\end{equation}
%
$s.t.\ \yL{\IdxCentroid} = \TLU{i}{*}(x)\ \forall \x \in \XL{i}$. The critical point of this quadratic energy w.r.t. each $\yL{\IdxCentroid}$ has a closed form:
%
\begin{equation} \label{eq:vwb_dy}
\begin{split}
    \yLU{\IdxCentroid}{*}
    = \frac{\sum_{i = 1}^{\N} \int_{\RegionL{i, \IdxCentroid}} \x d\MUL{i}(\x)}{N\sum_{i = 1}^{N} \int_{\RegionL{i, \IdxCentroid}} d\MUL{i}(\x)} 
    \approx \frac{\sum_{i = 1}^{\N} \sum_{\x \in \RegionL{i, \IdxCentroid}} x \MUL{i}(\x)}{N\sum_{i = 1}^{\N} \sum_{\x \in \RegionL{i, \IdxCentroid}} \MUL{i}(\x)}, \nonumber
\end{split}
\end{equation}
%
which is the center of mass of its correspondence across all measures. The latter expression is the ``stochastic'' version.

\begin{figure}[t]
\begin{center}
\centerline{\includegraphics[width=0.9\columnwidth]{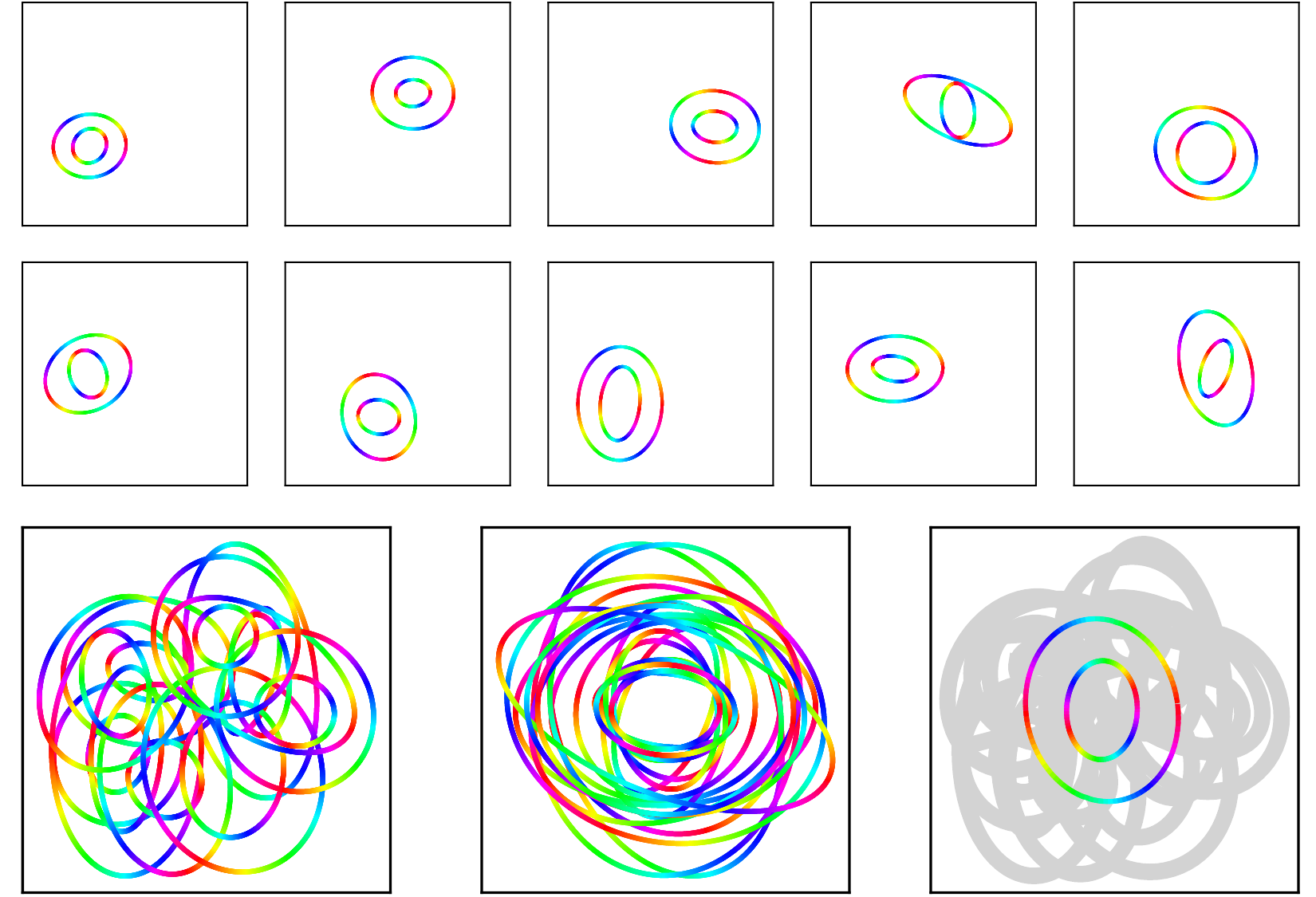}}
\caption{Ten random nested ellipses (top) averaged according to the Euclidean distance (left) and the Wasserstein distance (right) as implemented by VWB. For a better visual, we use the Euclidean sum instead. Middle is the Euclidean sum after re-centered. The VWB preserves the topology (rainbow colors) of the ellipses.}
\label{fig:vwb}
\end{center}
\end{figure}

The last step is to derive the update rule for the measure $\bm{\NU}$. \eqref{eq:wb3} is not differentiable w.r.t. $\bm{\NU}$. Still, we follow \cite{cuturi2014fast, mi2018regularized} and give the critical point and include the derivation in Appendix.
%
\begin{equation} \label{eq:vwb_dv}
\begin{split}
    \NUUL{*}{\IdxCentroid}
    = \frac{1}{N}\sum_{i = 1}^{N} \int_{\RegionLU{i, \IdxCentroid}{*}} d\MUL{i}(\x)
    \approx \frac{1}{\N} \sum_{i = 1}^{\N} \sum_{ \x \in \RegionLU{i, \IdxCentroid}{*}} \MUL{i}(\x), \nonumber
\end{split}
\end{equation}
%
where $\RegionLU{i, \IdxCentroid}{*} = \{\x\in\XL{i}\ |\ \|\x-\yL{\IdxCentroid}\|_{2}^{2} < \|\x-\yL{\IdxCentroidSecond}\|_{2}^{2}\ \forall \IdxCentroidSecond \neq \IdxCentroid\}$. $\NULU{\IdxCentroid}{*}$ coincides with the result of Lloyd's K-means algorithm in which the measure on each centroid accumulates all its assigned empirical measures.

Now that we have derived the rules for updating $\T$ and $\NU$, we summarize our algorithm for computing the VWB of a collection of measures $\{\MUL{i}\}_{i}$ in Appendix. As for the initial guess of the barycenter, if not specified, we can either run Lloyd's algorithm on all the measures as a whole and adopt the resulting $\K$ centroids or uniformly sample the space $\Y$. The choice of the measure on the centroids depends on the specific application. A ubiquitous choice is uniform Dirac measures, i.e. $\NUL{\IdxCentroid} = \frac{1}{\K} \delta[\yL{\IdxCentroid}]$. Figure~\ref{fig:vwb} suggests that by regarding the WD as the metric, we can find a mean shape on the same manifold, if there exists one.

Our method does converge since we follow coordinate descent and every step is convex \cite{grippo2000convergence}, given the assumption we made in~\ref{sec:ot} that $\X, \Y \subset \mathbb{R}^{n}$, $c(\x, \y)= \|\x - \y\|_{2}$, and $p = 2$. There are in total $\mathcal{O}(\K \cdot \N)$ variables for computing $N$ Monge maps $\{\TL{i}\}_{i=1}^{\N}$, and $\mathcal{O}(\K)$ variables as the support $\bm{\y}$ and $\mathcal{O}(\K)$ variables as the measure $\bm{\NU}$. We implemented VWB with PyTorch~\cite{paszke2019pytorch}. The code to reproduce the figures in this paper is at \href{https://github.com/icemiliang/pyvot}{https://github.com/icemiliang/pyvot}.\ignore{  Figure~\ref{fig} visualizes the gradient flow.}

\subsection{Metric Properties of (V)WBs}
\label{sec:metric_property}
In spite of extensive studies on metric properties of OT over the past century, the metric properties of Wasserstein barycenters have yet been fully explored. Some pioneer work includes~\cite{papadakis2019approximation,auricchio2018computing}. \ignore{ discussed a special case for two marginals, namely \textit{2-metric}.}

However, most of them focus on the barycenter of two measures ($\N=2$). We show in the following that WBs in general ($\N \geq 2$) induce a \textit{generalized metric} (\textit{n-metric}). First, let us define the total Wasserstein distance between the barycenter and all the marginal Borel measures:
%
\begin{equation} \label{eq:wb_nmetric}
    \WBL{\NU}(\MUNL{1:N}) \eqdef \underset{\NU \in \MUS(\Y)}{\inf} \frac{1}{N}\sum_{i = 1}^{N} \WUL{2}{2}(\MUL{i}, \NU),
\end{equation}
%
Then, we raise the following two propositions and prove them in Appendix.
\begin{proposition}\label{the:wb_nmetric}
$\WBL{\NU}(\MUNL{1:N})$ defines a generalized metric among $\{\MUL{i}\}_{i=1}^{\N}$, $N \geq 2$. Specifically, $\WBL{\NU}(\MUNL{1:N})$ satisfies the following properties. \\
1) Non-negativity: $\WBL{\NU}(\MUNL{1:N}) \geq 0$.\\
2) Symmetry: $\WBL{\sigma_1(1:\N)}(\NU) = \WBL{\sigma_2(1:\N)}(\NU)$, where $\sigma_1(1:N)$ and $\sigma_2(1:\N)$ are different permutations of the set ${1:\N}$.\\
3) Identity: $\WBL{\NU}(\MUNL{1:N}) = 0 \Longleftrightarrow \MUL{i} = \MUL{j}, \forall i \neq j$.\\
4) Triangle inequality: $\WBL{\NU}(\MUNL{1:N}) \leq \sum_{i=1}^{\N} \WBL{\NU}(\MUNL{1:\N+1 \backslash i})$.
\end{proposition}
\begin{proposition}\label{the:wb_nmetric_bigon}
The bound of the triangle inequality in Proposition~\ref{the:wb_nmetric} can be tightened by a linear factor. Specifically, we have $(N/2)\ \WBL{\NU}(\MUNL{1:N}) \leq \sum_{i=1}^{\N} \WBL{\NU}(\MUNL{1:\N+1 \backslash i})$.
\end{proposition}

The VWB $\NU = \sum_{\IdxCentroid=1}^{\K}\NUL{j}\delta[\yL{\IdxCentroid}] \in \MUS(\Y)$, as a special case of WBs, certainly inherits the metric properties because there is not such a restriction on the continuity of $\Y$. If we denote the total WD for the VWB with $\VWBL{\NU}(\MUNL{1:\N})$
%
%
, then we have:
\begin{corollary}\label{the:vwb_nmetric}
\ignore{As a special case of Proposition~\ref{the:wb_nmetric}, }$\VWBL{\NU}(\MUNL{1:\N})$ induces an n-metric over all $\MUL{i}$'s. In particular, the equal signs in 1) non-negativity and 4) inequality hold if and only if all $\MUL{i}$'s and $\NU$ have the same number of supports with positive Dirac measures $|\MUL{i}| = |\NU| = \K,\ \forall i \in \{1,...,\N\}$. 
\end{corollary}

\begin{figure}[t]
\begin{center}
\centerline{\includegraphics[width=0.9\columnwidth]{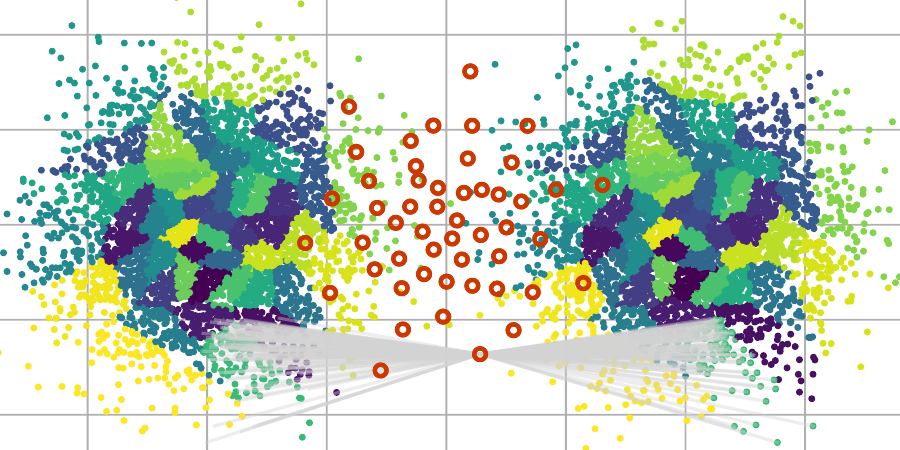}}
\caption{Transshipment: transporting measures through a set of discrete relays. Colors on the measures specify correspondences.}
\label{fig:ship1}
\end{center}
\end{figure}

\subsection{Approximate WDs with VWBs -- Transshipment}
\label{sec:transship}

We consider the transshipment problem as finding a Monge map from the source to the target that passes through a relay measure in the middle (see Figure~\ref{fig:ship1}). We solve it by VWBs. Our discussion comes directly from the conclusions in~\ref{sec:metric_property}:
\begin{corollary}\label{the:vwb_2metric}
As a special case of Corollary~\ref{the:vwb_nmetric}, $\VWBL{\NU}(\MUNL{1:2})$ induces a (2-)metric between $\MUL{1}$ and $\MUL{2}$. It is lower-bounded by $\frac{1}{4}\WLU{2}{2}(\MUL{1},\MUL{2})$ when \ignore{$\MUL{1}$, $\MUL{2}$, and $\NU$ have the same number of supports}$|\MUL{i}| = |\NU| = K$.
\end{corollary}

Appendix reveals the proof. \ignore{It is inspired by~\cite{papadakis2019approximation}. }Then, we can use a VWB to connect two measures and regard the total WD as an approximation to the true WD between them. We name it the \textit{variational Wasserstein distance}, or \textit{VWD}:
%
\begin{equation}\label{eq:vwb_ship}
\begin{split}
    \VWLU{2}{2}(\MUL{1}, \MUL{2}) &\eqdef 4 \VWBL{\NU}(\MUNL{1:2})\\ 
    & =\underset{\NU \in \MUS(\Y)}{\inf} 2\ \WUL{2}{2}(\MUL{1}, \NU) + 2\ \WUL{2}{2}(\MUL{2}, \NU). \nonumber
\end{split}
\end{equation}
%
We use the toy data above to evaluate the approximation against the number of supports, $K$. The two Gaussian measures share the same covariance matrix; their means differ by 1. Thus, the true WD is $1$. We use the results from linear programming (LP) and Sinkhorn algorithms for reference. Figure~\ref{fig:ship2} shows that VWD is still accurate with few supports. For each number of supports in the experiments, we run our algorithm 10 times with different random initial locations. We draw the error band with light color. Until $1500$ supports, ratio $0.3$, our algorithm produces stable approximations that have almost zero variance. 
\ignore{\cite{papadakis2019approximation} shows that, if we marginalize a Kantorovich map to the interpolation of the supports of the two marginals, the resulting distance equals the original WD. \ignore{Yet, \ignore{Papadakis and we approach approximating WDs by solving the optimal transshipment problem differently: }Papadakis targets a Kantorovich map; we solve for a Monge map. }We discuss a similar case for Monge maps in Appendix.}

\begin{figure}[t]
\begin{center}
\centerline{\includegraphics[width=0.9\columnwidth]{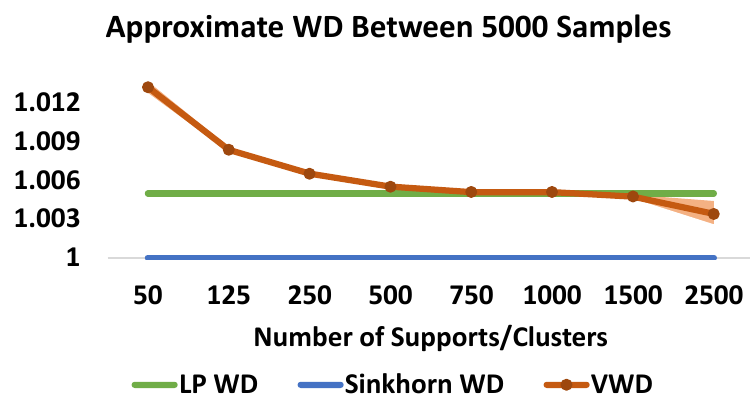}}
\caption{WDs between two Gaussian's vs. number of supports.}
\label{fig:ship2}
\end{center}
\end{figure}

\subsection{On Unbalanced Measures}
\label{sec:unbalanced_ot}

When measures are not probabilities or, more generally their integrals do not equal, we are solving \textit{unbalanced OT}. \cite{benamou2003numerical} first explored the problem. Researchers since then have offered various formulations and perspectives to approach it, e.g. \cite{liero2018optimal} adding $f$-divergences as regularizers instead of constraints on the marginals. Here, we discuss VOT and VWBs for unbalanced measures. Without loss of generality, let us assume $\int_{\X}d\MU(\x) = w,\ \sum_{\IdxCentroid = 1}^{\K} \NUL{\IdxCentroid} = 1$. We denote the mass in each power Voronoi cell by $w_{\RegionL{\IdxCentroid}} = \int_{\RegionL{\IdxCentroid}}d\MU(\x)$.
Inspired by the discussion in~\cite{peyre2019computational}, we propose to penalize the quadratic mismatch of the mass for each cell $\IdxCentroid$.
%
%
%
\begin{equation} \label{eq:unbalanced_vot}
\begin{gathered}
\min_{\bm{\Region}}\ \EnergyL{7}[\bm{\Region}] \eqdef 
\int_{\X}  \|\x - T(\x)\|^2_2  d\MU(\x) + \lambda \sum_{\IdxCentroid = 1}^{\K}  \left( w_{\IdxCentroid} - \NUL{\IdxCentroid}\right)^2,
\end{gathered}
\end{equation}
%
s.t. $\sum_{\IdxCentroid} w_{\RegionL{\IdxCentroid}} = w$. If $\lambda \rightarrow \infty, w = 1$, \eqref{eq:unbalanced_vot} reverses to~\eqref{eq:monge}. 

In the following, we discuss \eqref{eq:unbalanced_vot} in two cases: $\NUL{\IdxCentroid} = \frac{1}{\K}$ and a more general one, $\NUL{\IdxCentroid} \in (0,1),\ \sum_{\IdxCentroid = 1}^{\K} \NUL{\IdxCentroid} = 1$.

\begin{figure}[t]
\begin{center}
\centerline{\includegraphics[width=0.85\columnwidth]{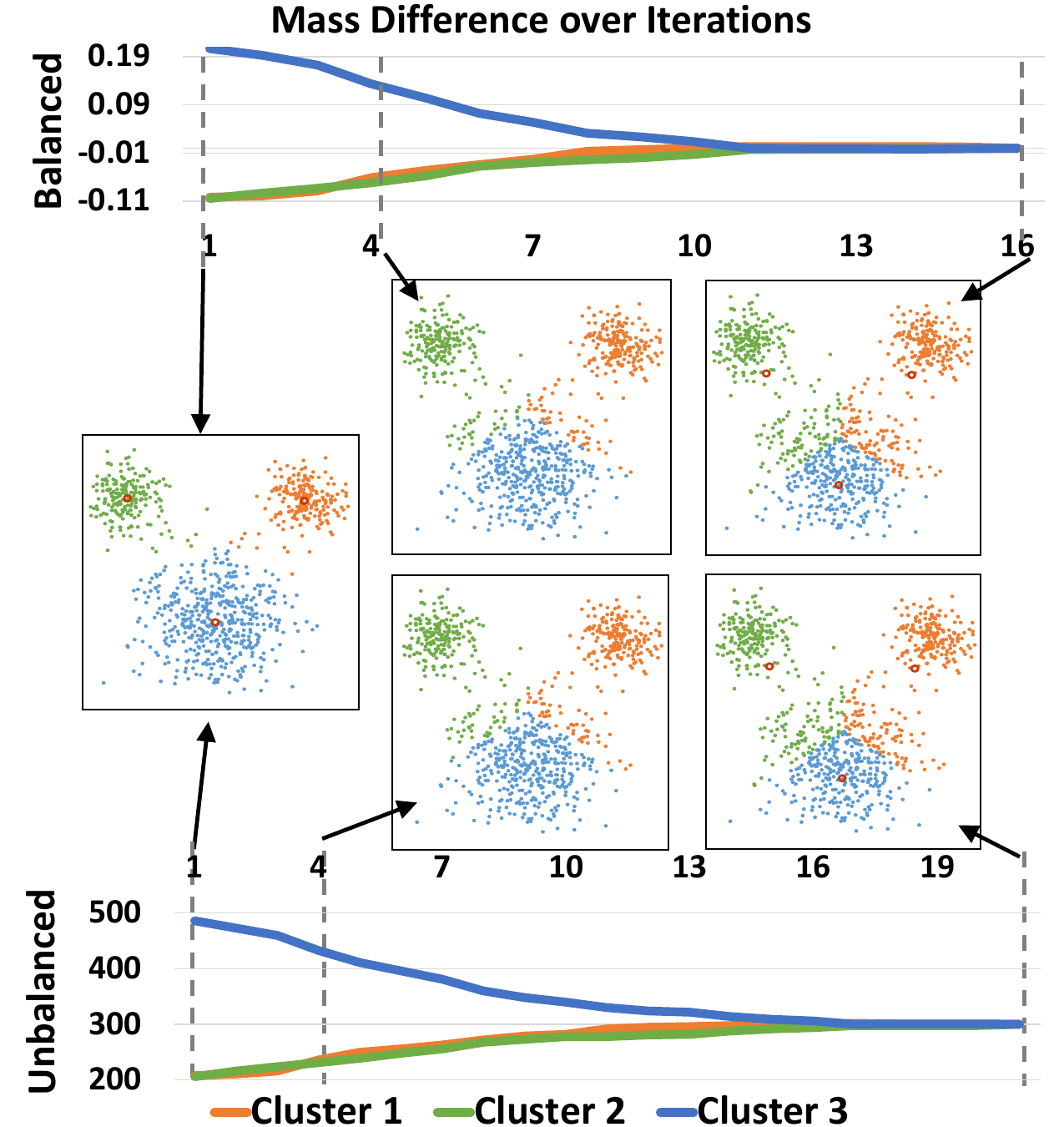}}
\caption{Mass difference over iterations for VOT on balanced and unbalanced measures. They follow the same trend and converge at almost the same rate. The resulting clusters are exactly the same.}
\label{fig:uot}
\end{center}
\end{figure}

\textbf{Case 1:} $\NUL{\IdxCentroid} = \frac{1}{\K}$. 
It is trivial to verify that minimizing the second term in~\eqref{eq:unbalanced_vot} over $\bm{\Region}$ under its equality constraint yields all $w_{\RegionL{\IdxCentroid}}$'s equal to each other, i.e. $w_{\RegionL{\IdxCentroid}} = \frac{1}{\K}w$. On the other hand, the gradient of the VOT energy~\eqref{eq:vot} has the form $\int_{\RegionL{\IdxCentroid}} d\MU(\x) - \NUL{\IdxCentroid} \equiv w_{\RegionL{\IdxCentroid}} - \frac{1}{\K}w$. The question now is whether $w_{\RegionL{\IdxCentroid}} = \frac{1}{\K}w$ minimizes~\eqref{eq:vot}. If so, then we can instead solve~\eqref{eq:vot} to minimize the second term in~\eqref{eq:unbalanced_vot}.

When $w_{\RegionL{\IdxCentroid}} = \frac{1}{\K}w$, the gradient $\nabla\EnergyL{2}[\bm{\h}] = \{w_{\RegionL{\IdxCentroid}} - \NUL{\IdxCentroid}\}_{\IdxCentroid}$ becomes constant and thus $\bm{\h}$ is being translated at a constant rate. Certainly, translation does not modify a power Voronoi diagram as specified in~\eqref{eq:monge_ot_dual}. Therefore, $\EnergyL{2}[\bm{\h}]$ saturates\ignore{ to the point where $w_{\RegionL{\IdxCentroid}} = \frac{1}{\K}w$}. 
For any other partition such that $\exists\ \RegionL{\IdxCentroid}', w_{\RegionL{\IdxCentroid}'} \neq \frac{1}{\K}w$, we have
%
\begin{equation}
\begin{split}
\sum_{\IdxCentroid=1}^{\K} \int_{\RegionL{\IdxCentroid}} & \left( \|\x - \yL{\IdxCentroid}\|_{2}^{2} + \hL{\IdxCentroid} \right) d\MU(\x) \\
&\leq \sum_{\IdxCentroid=1}^{\K} \int_{\RegionL{\IdxCentroid}'} \left( \|\x - \yL{\IdxCentroid}\|_{2}^{2}  + \hL{\IdxCentroid} \right) d\MU(\x).
\nonumber
\end{split}
\end{equation}
%
Therefore, $w_{\RegionL{\IdxCentroid}} = \frac{1}{\K}w$ indeed minimizes~\eqref{eq:vot}. Meanwhile, we know that an unweighted Voronoi diagram ($\hL{\IdxCentroid} = 0$) would minimize the first term in~\eqref{eq:unbalanced_vot}. Thus, we can directly give the solution to \eqref{eq:unbalanced_vot} as $\{\frac{\lambda}{1 + \lambda}\hL{\IdxCentroid}\}_{\IdxCentroid=1}^{\K}$.

%
%

%


%

\textbf{Case 2:} $\NUL{\IdxCentroid} \in (0,1),\ \sum_{\IdxCentroid = 1}^{\K} \NUL{\IdxCentroid} = 1$.
It is also trivial to verify that minimizing the second term in~\eqref{eq:unbalanced_vot} over $\bm{\Region}$ yields $w_{\RegionL{\IdxCentroid}} = \NUL{\IdxCentroid}w$ (replace $\frac{1}{K}$ with $\NUL{\IdxCentroid}$). $w_{\RegionL{\IdxCentroid}} = \NUL{\IdxCentroid}w$ also triggers the convergence of VOT as in Case 1.

At this point, we claim that VOT, \eqref{eq:vot}, minimizes the total transportation cost regardless of the measures equal or not. We leave rigorous proofs to future work.  We illustrate the convergence in Figure~\ref{fig:uot}. The top half shows VOT between balanced measures and the bottom half shows unbalanced measures, $w = 900 = 500 + 2 \times 200, \NUL{\IdxCentroid} = \frac{1}{3}$. Note that the gradient of the VOT and VWB, \eqref{eq:vwb_dh}, correlates to the absolute measure values. Thus, we should scale the step size, $\eta$ in \eqref{eq:vwb_gd}, for each VOT according to the difference of the measure, i.e. $\eta_i / w $, assuming the total for $\NU$ is $1$. Figure~\ref{fig:uot} shows that under the same (scaled) GD step size, VOT in two cases follows the same trend.

We apply VWBs to unbalanced measures and show in Figure~\ref{fig:uvwb} the resulting barycenter of two Gaussian's of different samples, $5$k vs. $1$k. We choose $\lambda = \infty$ in \eqref{eq:unbalanced_vot}. We can also see that Monge maps are absolutely \textit{binary} and \textit{sparse}.

\begin{figure}[t]
\begin{center}
\centerline{\includegraphics[width=0.9\columnwidth]{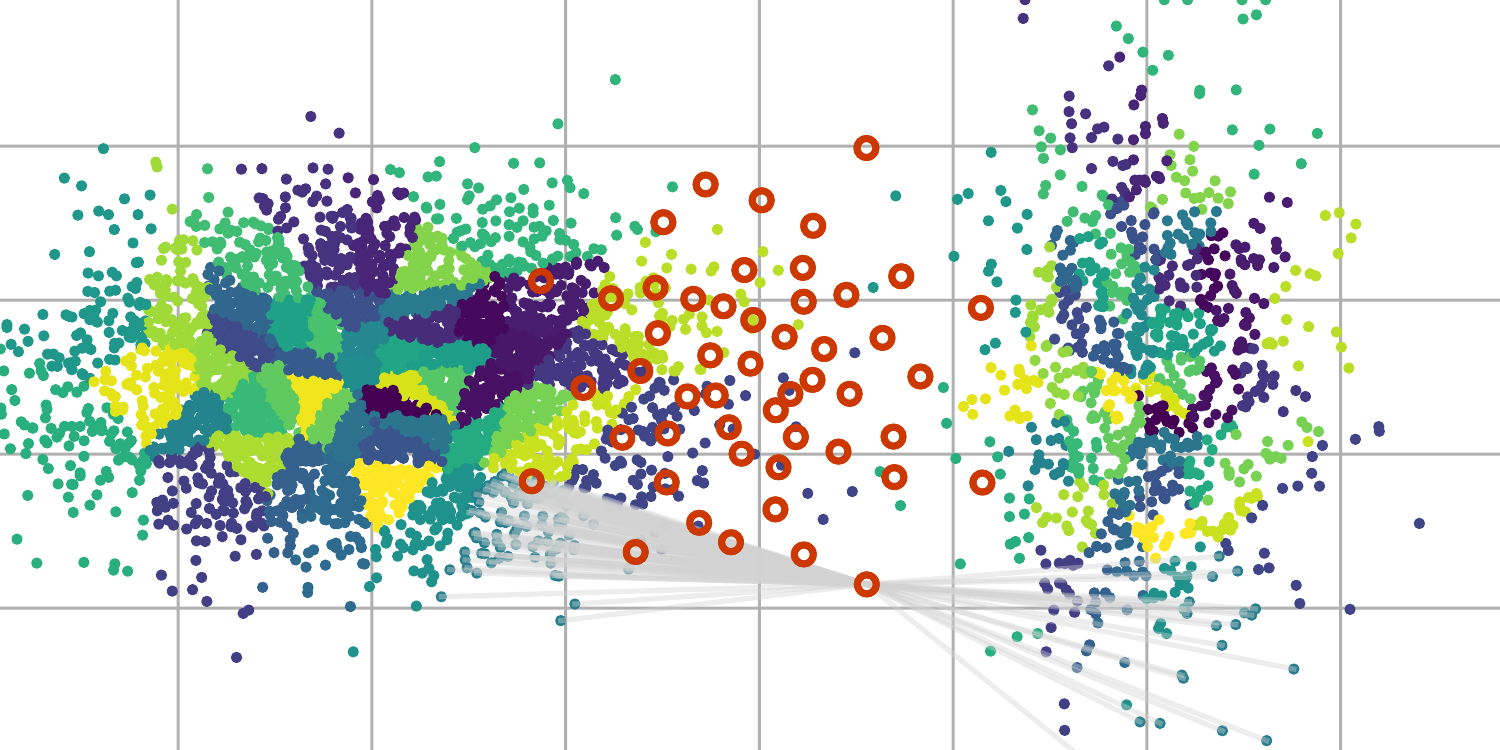}}
\caption{Interpolating two Gaussian's of different number of samples by computing the VWB results in a mean isotropic Gaussian.}
\label{fig:uvwb}
\end{center}
\end{figure}

\subsection{On Spherical Domains}
\label{sec:svwb}

Optimal transport on geometric domains other than the Euclidean domain extends its applications~\cite{solomon2015convolutional,staib2017parallel,cui2019spherical}. \cite{cui2019spherical} relates \textit{spherical power Voronoi diagram} to OT on unit spheres. Inspired by that, we study our VWB on spherical domains and its metric properties.

Let us define a new ground metric on a unit sphere, $\mathbb{S}^2 \times \mathbb{S}^2 \rightarrow \mathbb{R}^{\geq 0}$, as $\Dist(\x, \yL{\IdxCentroid}) = -\ln\langle\x , \yL{\IdxCentroid}\rangle$ and the OT distance:
\begin{equation}
    \WL{1}' = \underset{\T \in \PSL{T}(\MU,\NU)}{\inf} \EnergyL{8}[\PI] \eqdef -\int_{\mathbb{S}^2} \ln\langle x , T(x) \rangle d\MU(\x)
\end{equation}
s.t. $\int_{\mathbb{S}^2}(\psi \circ T)d\MU(x) = \int_{\mathbb{S}^2}\psi d\NU(y)$ for all non-negative $\psi$.

\begin{figure}[t]
\begin{center}
\centerline{\includegraphics[width=0.7\columnwidth]{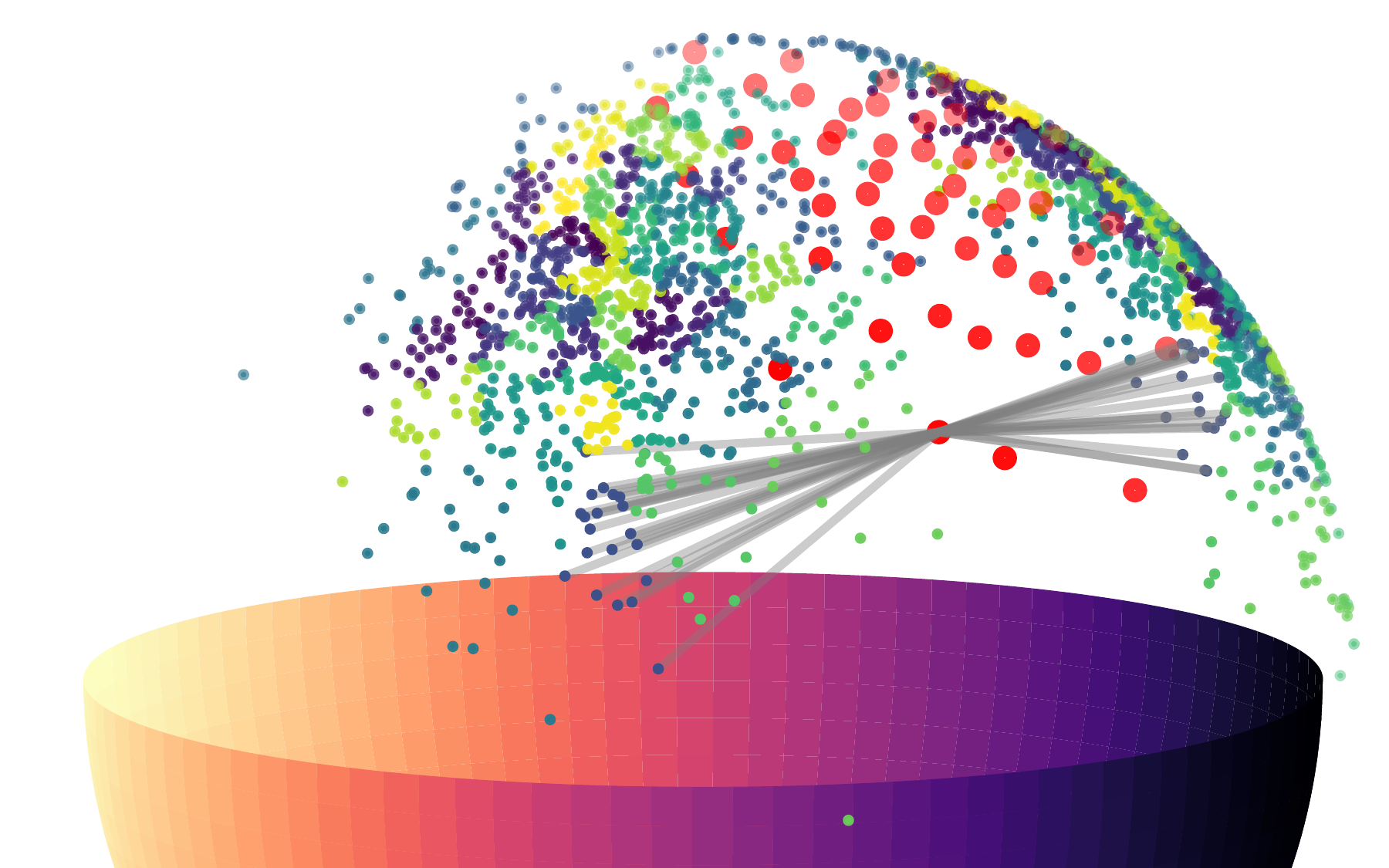}}
\caption{Interpolating two Gaussian distributions on a sphere w.r.t. the VWD. By using the VWB, we can build sparse connection between the two domains via a few discrete relays.}
\label{fig:sphere}
\end{center}
\end{figure}

Following \cite{cui2019spherical}, we define the \textit{power distance} on a sphere as $c'(\x, \yL{\IdxCentroid}) = -\ln\langle \x , \yL{\IdxCentroid} \rangle / \cos r_{\IdxCentroid}$ and thus the power Voronoi diagram in the spherical domain  $\RegionL{\IdxCentroid} \eqdef \{\x \in \mathbb{S}^{2}\ | c'(\x, \yL{\IdxCentroid}) \leq c'(\x, \yL{\IdxCentroidSecond}), \forall \IdxCentroidSecond \neq \IdxCentroid\}$. $r_{\IdxCentroid}$ is the weight of each power cell, it relates to the VOT minimizers by $\cos{r} = e^{h}$. Then, the derivation in \ref{sec:vot} gives us the Monge map.

$-\ln\langle\x , \yL{\IdxCentroid}\rangle$ does not satisfy triangle inequality but the other three metric properties. Thus, $\WL{1}'$ inherits those properties. We notice that the proof for Proposition~\ref{the:wb_nmetric} does not require triangle inequality. Therefore, the n-metric properties still hold for the barycenter w.r.t. $\WL{1}'$.
%
$$\WBL{1:\N}'(\NU) \eqdef \underset{\NU \in \MUS(\Y)}{\inf} \frac{1}{\N}\sum_{i = 1}^{\N} \WL{1}'(\MUL{i}, \NU)$$
%
%

Although $\WL{1}'$ is not a true metric, we can still find a ``mean'' of multiple marginals by alternatively minimizing the total ``distance'' as in~\ref{sec:wbvot}. Figure~\ref{fig:sphere} shows an example where the VWB simultaneously partitions two domains on the sphere. For simplicity, we draw connections with straight lines.

\section{Geometric Clustering via VWBs}
\label{sec:gcvwb}
In this section, we further connect VWBs to several clustering problems. We consider a fixed number of clusters, $K$, the quadratic Euclidean distance as the ground metric, and mainly the spatial relation between samples. We refer to this scenario as \textit{geometric clustering}. From now on, we discretize the measures: $\NU = \sum_{\IdxCentroid = 1}^{K}\NUL{\IdxCentroid}\delta[\yL{\IdxCentroid}], \MUL{i} = \sum_{j=1}^{n_i}\MU(\xL{j})\delta[\xL{j}]$ and assume that $n_i \gg K,\ \forall i$.

\subsection{Regularized K-Means Clustering}
\label{sec:regularized_kmeans}

In light of the discovery of VWBs for unbalanced measures in~\ref{sec:unbalanced_ot}, we now introduce a relaxed version of the constrained K-means problem. We call it \textit{regularized K-means}.

The classic K-means problem has the objective as follows:
%
\begin{equation} \label{eq:kmeans}
    \min_{\bm{\Region}}\sum_{\IdxCentroid = 1}^{K}\sum_{\x \in \RegionL{\IdxCentroid}}\|\x - \yL{\IdxCentroid}\|_{2}^{2},\ \ \yL{\IdxCentroid} = \frac{1}{|\RegionL{\IdxCentroid}|}\sum_{\x \in \RegionL{\IdxCentroid}} \x,
\end{equation}
%
where $|\RegionL{\IdxCentroid}|$ is the number of samples supported in $\RegionL{\IdxCentroid}$. By adding the marginal constraint $\NUL{\IdxCentroid} = \sum_{\x \in \RegionL{\IdxCentroid}}\MU(\x)$ with pre-defined, fixed measures $\{\NUL{\IdxCentroid}\}_{\IdxCentroid = 1}^{K}$, we turn~\eqref{eq:kmeans} into the \textit{constrained K-means} problem~\cite{bradley2000constrained,cuturi2014fast}, or the \textit{Wasserstein Means} problem coined in~\cite{ho2017multilevel}. As discussed in Section~\ref{sec:unbalanced_ot}, when the total measures do not equal, such constraints instead become regularizers. Then, we define the objective of the regularized K-means clustering problem as:
%
\begin{equation} \label{eq:regularized_kmeans}
    \min_{\bm{\Region}, \bm{\y}}\sum_{\IdxCentroid = 1}^{\K}\sum_{\x \in \RegionL{\IdxCentroid}}\|\x - \yL{\IdxCentroid}\|_{2}^{2} + \lambda \sum_{\IdxCentroid = 1}^{\K} \left( \NUL{\IdxCentroid} - w_{\IdxCentroid} \right)^{2},\
\end{equation}
%
where $w_{\IdxCentroid} = \sum_{\x \in \RegionL{\IdxCentroid}}\MU(\x)$. If $\lambda = 0$, \eqref{eq:regularized_kmeans} becomes K-means; if $\lambda \rightarrow \infty$, \eqref{eq:regularized_kmeans} becomes Monge OT. As practiced in~\cite{cuturi2014fast,mi2018variational}, we can alternatively solve for $\bm{\Region}$ and $\yL{\IdxCentroid} = 1/|\RegionL{\IdxCentroid}|\sum_{\x \in \RegionL{\IdxCentroid}} \x$. The energy~\eqref{eq:regularized_kmeans} will monotonically decrease and eventually converge into a cycle of one. Figure~\ref{fig:regularized_kmeans} illustrates the regularized K-means result which informally looks like an interpolation between K-means and constrained K-means.

\ignore{
\begin{corollary}
*** minimizes ~\eqref{eq:regularized_kmeans} blah blah blah blah blah blah blah blah blah blah blah blah blah blah blah blah blah
\end{corollary}}

\begin{figure}[t]
\begin{center}
\centerline{\includegraphics[width=\columnwidth]{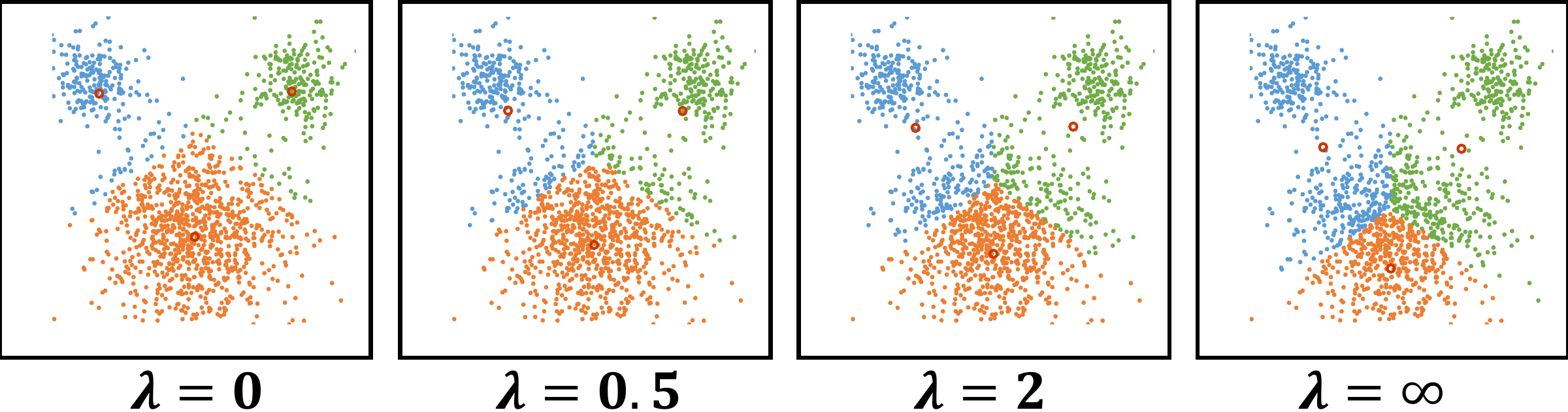}}
\caption{Results from different regularization strength $\lambda$ in \eqref{eq:regularized_kmeans}. Left is traditional K-means and right is constrained K-means.}
\label{fig:regularized_kmeans}
\end{center}
\end{figure}

\subsection{Co-clustering Spatial Features in $\mathbb{R}^{n}$}
\label{sec:co_cluster}
Extending the Wasserstein clustering procedure to multiple targets induces the co-clustering problem. In this section, we discuss the connection between co-clustering problems and VWBs. In particular, because we use quadratic Euclidean distances as the ground metric, we focus on co-clustering spatial features embedded in the Euclidean space.

Given multiple distributional data, the goal of co-clustering is to simultaneously partition each domain to 1) minimize the pairwise variance in the same cluster and 2) minimize the pairwise variance for each cluster across domains. We assume all samples reside in $\mathbb{R}^{n}$ equipped with $\|\cdot\|_{2}$, then:
%
\begin{equation}
\begin{split}
    \underset{\RRegionL{i}}{\min}\ \EnergyL{9}[\RRegionL{i}] \eqdef & \sum_{i=1}^{\N} \sum_{\IdxCentroid=1}^{\K} \frac{1}{2|\RegionL{i,\IdxCentroid}|} \sum_{\x,\x' \in \RegionL{i,\IdxCentroid}} \|\x - \x'\|_{2}^{2} \\
    + \sum_{\IdxCentroid=1}^{\K}  \sum_{1 \leq i < j \leq \N} & \frac{\lambda_{i,j,\IdxCentroid}}{|\RegionL{i,\IdxCentroid}|+|\RegionL{j,\IdxCentroid}|} \sum_{\substack{\x \in \RegionL{i,\IdxCentroid} \\ \x' \in \RegionL{j,\IdxCentroid} }} \|\x - \x'\|_{2}^{2}. \nonumber
\end{split}
\end{equation}
%
where $|\RegionL{i,\IdxCentroid}|$ is the number of samples in $\RegionL{i,\IdxCentroid}$; $\lambda_{i,j,\IdxCentroid} \in \{0, 1\}$ specifies the correspondence of the clusters across different domains. Thus, $\sum_{i} \lambda_{i,j,\IdxCentroid} = 1$ and $\sum_{j} \lambda_{i,j,\IdxCentroid} = 1$. As for K-means, we can simplify the pairwise variance with the mean of each cluster at each domain, $\alpha_{i, \IdxCentroid}$:
%
\begin{equation} \label{eq:cocluster}
\begin{split}
    \underset{\RRegionL{i}}{\min}\ \EnergyL{9}[\RRegionL{i}] \equiv & \sum_{i=1}^{\N} \sum_{\IdxCentroid=1}^{\K} \sum_{\x\in \RegionL{i,\IdxCentroid}} \|\x - \alpha_{i, \IdxCentroid}\|_{2}^{2} \\
    + \sum_{\IdxCentroid=1}^{\K} &  \sum_{i = 1}^{\N} \sum_{j \neq i} \lambda_{i,j,\IdxCentroid} \sum_{\x \in \RegionL{i,\IdxCentroid}} \|\x - \alpha_{j, \IdxCentroid}\|_{2}^{2}.
\end{split}
\end{equation}
%
where $\alpha_{i, \IdxCentroid} = \frac{1}{|\RegionL{i, \IdxCentroid}|} \sum_{\x \in \RegionL{i, \IdxCentroid}} \x$\ignore{ $\alpha_{i, \IdxCentroid} = 1 / |\RegionL{i, \IdxCentroid}| \sum_{\x \in \RegionL{i, \IdxCentroid}} \x$} is the cluster center for each cluster at each domain. The first term of \eqref{eq:cocluster} is solving $\N$ K-means problems. The second term is solving $\N(\N-1)$
K-means problems but with the cluster centroids at other domains. Thus, we can further simplify the problem into:
%
\begin{equation} \label{eq:cocluster2}
\begin{split}
    \underset{\bm{\RegionL{i}}}{\min}\ \EnergyL{9}[\bm{\RegionL{i}}] \equiv \ignore{&
    \sum_{i=1}^{\N} \sum_{j=1}^{\N} \sum_{\IdxCentroid=1}^{\K} \sum_{\x \in \RegionL{i,\IdxCentroid}} \|\x - \alpha_{j, \IdxCentroid}\|_{2}^{2}\\}
    &     \sum_{i=1}^{\N}  \sum_{\IdxCentroid=1}^{\K} \sum_{\x \in \RegionL{i,\IdxCentroid}} \sum_{j=1}^{\N} \|\x - \alpha_{j, \IdxCentroid}\|_{2}^{2}\\
\end{split}
\end{equation}

Solving \eqref{eq:cocluster2} involves alternatively updating partition $\{\bm{\RegionL{i}}\}_{i}$ and the centroid $\{\alpha_{i, \IdxCentroid}\}_{i, \IdxCentroid}$. When updating $\{\bm{\RegionL{i}}\}_{i}$ with fixed $\{\alpha_{i, \IdxCentroid}\}_{i, \IdxCentroid}$, we can rewrite \eqref{eq:cocluster2} as 
%
\begin{equation} \label{eq:cocluster3}
\begin{split}
    \EnergyL{11}[\RRegionL{i}]
    = &     \sum_{i=1}^{\N}  \sum_{\IdxCentroid=1}^{\K} \sum_{\x \in \RegionL{i,\IdxCentroid}} \left[\x - \left[\sum_{j=1}^{\N}  \alpha_{j, \IdxCentroid}\right] \right]^2 + C \\
    \eqdef &     \sum_{i=1}^{\N}  \sum_{\IdxCentroid=1}^{\K} \sum_{\x \in \RegionL{i,\IdxCentroid}} \left(\x - \hat{\alpha_{\IdxCentroid}} \right)^2 + C.
\end{split}
\end{equation}
%
$C$ is some constant. Thus, we convert co-clustering to $N$ $K$-means problems with the same set of centroids.

Then, we can naturally impose a constraint on the weights, i.e $\int_{\RegionL{i, \IdxCentroid}}d\MUL{i}(\x) = \NUL{\IdxCentroid},\ \forall i,\ \IdxCentroid$, to turn the problem into a VWB problem which is also an $N$ constrained K-means problem. Note, that it is trivial to extend it into a generalized VWB problem, by instead inserting the weighted constraint into the main objective as we did in \ref{sec:regularized_kmeans}.

\subsubsection{Regularized VWBs for Co-Clustering} \label{sec:regularized_vwb}

In addition to purely clustering feature domains according to Wasserstein losses, we can regularize the correspondences based on prior knowledge. Inspired by \cite{alvarezmelis2019towards,mi2018regularized}, we regularize the correspondence by global invariances. Directly regularizing Monge correspondences is highly intractable because Monge maps are basically binary permutations and thus not differentiable. Therefore, we instead regularize the centroid update process. 

To this end, instead of using the average of the centroids as we did in \eqref{eq:cocluster3}, we estimate the rigid transformation (isometry) between the VWB and the centroids of each domain by minimizing $\|\bm{\y} - H_{i}\bm{\alpha}_{i}\|_{2}^{2}\ \forall i$, subject to $H_i$ composing a rotation and a translation, i.e. $H_i = [R_i|t_i]$. This can be done by singular value decomposition (SVD) with minimum computational costs. After that, we average all the transformations by separately averaging rotations and translations. With the abuse of notation, we simply denote the process by $\widetilde{H} = \frac{1}{N} \sum_{i} H_i$, but as we know, we need to factorize the rotations into quaternions before averaging them. The final location for the supports $\bm{\y}$ is given by $\widetilde{\bm{\y}} = \widetilde{H}\bm{\y}$.

\subsection{Vector Quantization and Data Compression} \label{sec:compress}

Lloyd's K-means algorithm was initially proposed for vector quantization and has been a fundamental choice for data compression. It centers at using fewer samples to approximate the entire distribution. In light of the connection between VWBs and K-means, we raise the problem of compressing multiple distributional data as a whole with Wasserstein barycenters and propose the VWB as a natural choice. It shares the same objective as the WB. Intuitively, we use sum of WDs to measure the compression error.

By using VOT, we obtain a surjection from each domain to the barycenter. Because we optimize over the height vector $\hhL{i}$ \eqref{eq:vwb_gd}, given empirical samples and the barycenter, we can fully recover the surjection by only using $\hhL{i}$ at the negligible expense of computing the power distance as in \eqref{eq:monge_ot_dual}. In this way, for a barycenter of size $\K$ of $\N$ empirical distributions each having $M$ samples, we reduce the storage burden from $\mathcal{O}(\N M\K)$, as it would be for Sinkhorn distance-based methods, to $\mathcal{O}(\N\K)$. This is particularly useful when $M$ is large and when we need to store multiple interpolations between marginals. 

Furthermore, with the VWB, we do not even need the original distributions to parameterize the compression maps because our method is based on the geometry of the data and given the height vector $\hhL{i}$ and barycenter supports $\bm{\y}$ we can uniquely partition each original domain with a power Voronoi diagram $\RRegionL{i}$; or, equivalently, the graph of the piece-wise linear function $\theta_{\bm{\h}}(\x) = \underset{\IdxCentroid}{\max} \{\x \yL{\IdxCentroid} + \hL{\IdxCentroid}\}$.

\section{Applications} \label{sec:exp}
We demonstrate the use of VWBs with point cloud interpolation and image compression.


\subsection{Point Cloud Interpolation with Global Invariance} \label{sec:point_cloud}
Shape interpolation is a typical application of Wasserstein barycenter techniques. We compute the barycenter that has the minimum weighted average WD to all the marginal shapes. When the marginals are congruence to each other, we can leverage the congruency to regularize the process to update the barycenter. We adopt the approach in \ref{sec:regularized_vwb} and compute the VWB that has the minimum VWD to two marginal shapes. The correspondences are regularized by a rigid transformation in order to preserve the global structure of the shape. 
Ideally, we can obtain a ``mean'' shape that lies at the middle of the marginals and the rotations to the marginals share the same angles but in opposite directions. Figure~\ref{fig:icp} shows the result that verifies our hypothesis.

In this experiment, we are given two Kittens off by an unknown rigid transformation. Our goal is to interpolate, by computing a regularized Wasserstein barycenter, a new Kitten in between that is rigid to the original Kittens and the amount of translation and rotation is linear to the weights of the two original Kittens.

The marginal Kittens each have $7,805$ sample points. We assume all the samples have equal weights. They are apart from each other by a rigid transformation composed by a random translation vector $t$ and a random rotation matrix $r$. In this example, they are as follows:
\begin{equation*}
  t = \begin{bmatrix} -1.97\\ -0.73\\ -0.30\end{bmatrix}
  \quad
  r = \begin{bmatrix} 0.87 &-0.23 & 0.44 \\ 0.41 & 0.84 & -0.36 \\ -0.30 & 0.49 & 0.82\end{bmatrix}
\end{equation*}

\begin{figure}[t]
\begin{center}
\centerline{\includegraphics[width=0.9\columnwidth]{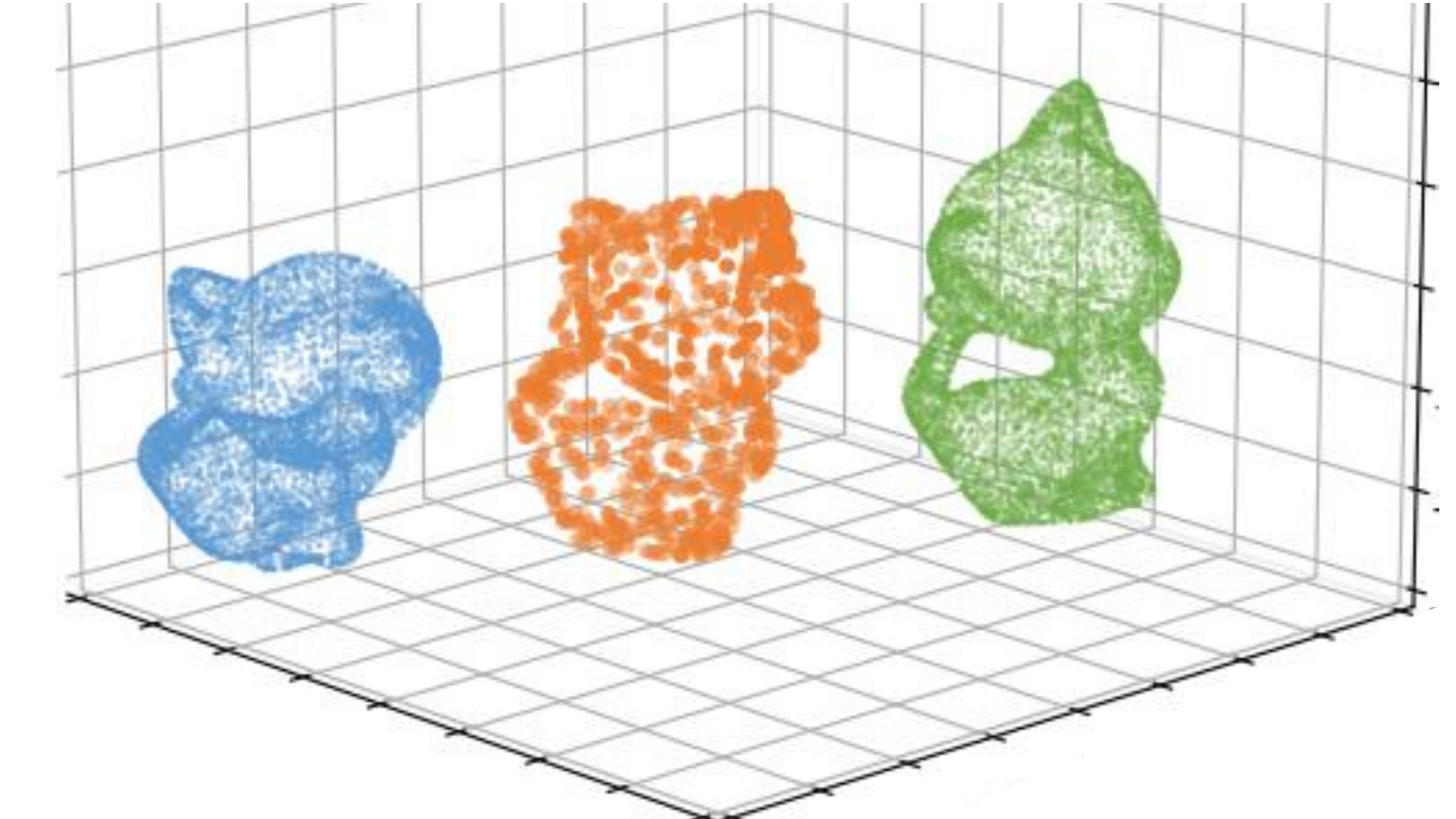}}
\caption{Point cloud interpolation that preserves global structures.}
\label{fig:icp}
\end{center}
\end{figure}

The barycenter Kitten w.r.t. the VWD (variational Wasserstein distance) has $780$ supporting Dirac measures. The regularization strength, $\lambda$, is $10$. One of the post-processing options to transport all the samples from the marginals is that for each sample find its nearest 3 or more cluster centers and use inverse barycenter coordinates to find its new location on the target Kitten in the middle.

\subsection{Image Compression}
\label{sec:exp_partition}

We demonstrate the use of our method for data compression by quantizing the RGB colors of an image into a fixed number of clusters. See Figure~\ref{fig:color} for the results. The top row shows the original images of dimension $128^2 \times 3$. We embed all the pixels into the RGB color space $\X = \{x \in \mathbb{R}^{3}\ |\ \|x\|_{\infty} \leq 1 \}$. Our goal is to compute, for example, $K=16$ centroids that partition all the pixels into their clusters. In this way, we compress the storage for each pixel from $24$ bits to $4$ bits. The second row in Figure~\ref{fig:color} shows resulting images of using Lloyd's K-means(++) algorithm, and the third row shows the results of using our VOT solver. Compared to Lloyd's, VOT well distributes the centroids into the pixel domain, resulting in a smoother transition from color to color. The correspondences in the color space we show in Appendix also confirm this. Finally, we simultaneously merge and compress the colors from all three images by using VWB. The last row shows the resulting images sharing the same color distribution that only consists of 16 discrete centroids. It has the same $\WL{2}$ to each original color distribution (marginal). In Appendix, we further show the results that comes from the centroids having different $\WL{2}$'s to each marginals, i.e. $\lambda_{i} \neq \frac{1}{N}$ in \eqref{eq:wb_nmetric}. We show the RGB color distribution of each image in Appendix.

\begin{figure}[t]
\begin{center}
\centerline{\includegraphics[width=0.95\columnwidth]{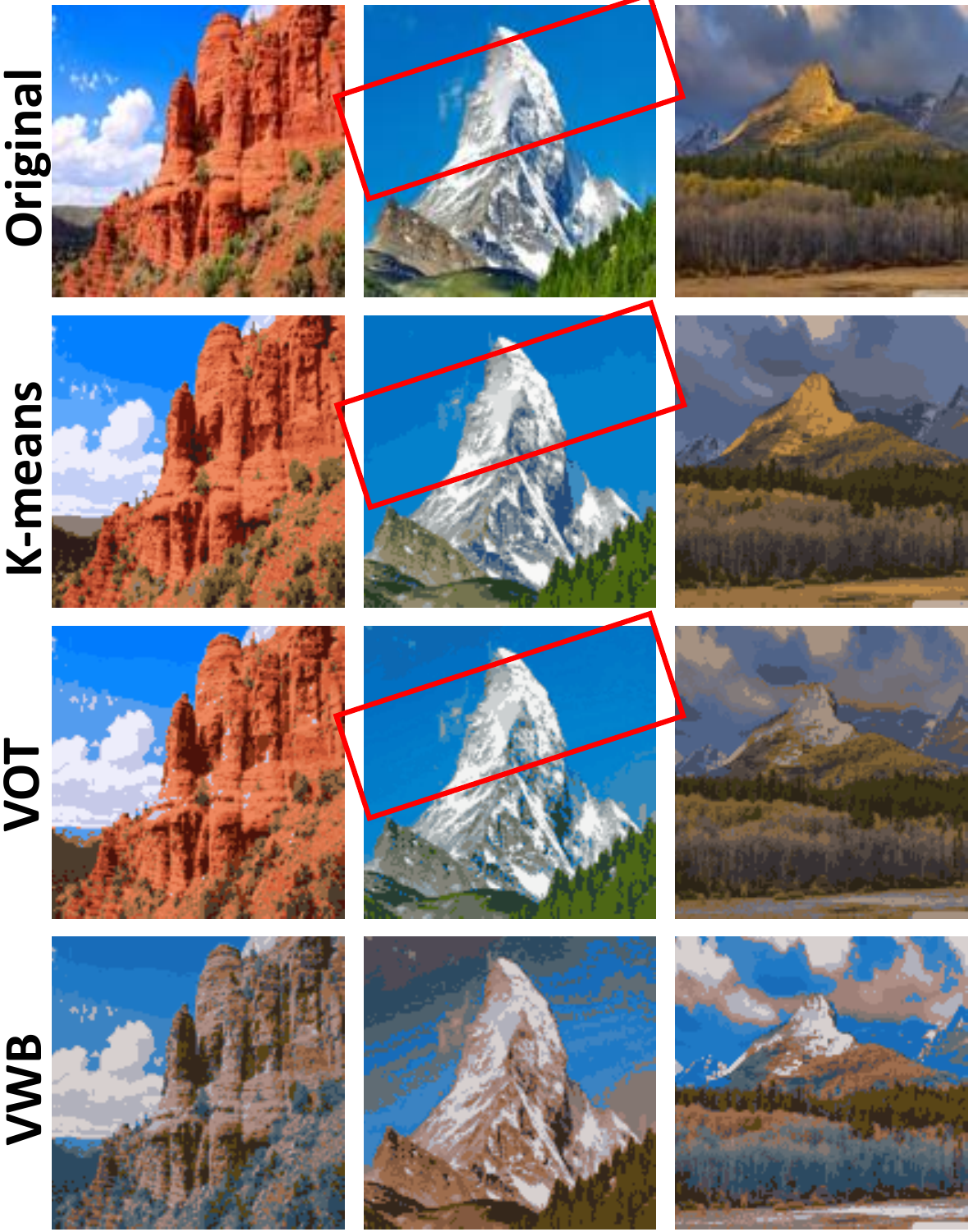}}
\caption{Quantizing RGB values from $24$ bits to $4$ bits by solving K-means, OT, and the WB. Solving OT results in smoother images; solving WBs can cluster and merge colors at the same time.}
\label{fig:color}
\end{center}
\end{figure}



\section{Discussion}
\label{sec:discuss}

We conclude by discussing the advantages and disadvantages of VWBs and several future directions.

Algorithms solving K-means like clustering problems are in general sensitive to initial choices. Common solutions include using a subset of samples and spreading the seeds across the domain, e.g., K-means++. We tried the results from K-means++ as the initial choice for our barycenters and also tried a pre-defined Gaussian distribution whose mean is the average of the means of the marginals as prior knowledge. We did not find visible differences.

Monge maps between discrete measures may not exist, e.g. transporting 3 Dirac points $\{\frac{1}{3}\delta[\xL{j}]\}_{j=1}^{3}$ to 2 Dirac points $\{\frac{1}{2}\delta[\yL{j}]\}_{\IdxCentroid=1}^{2}$. In this case, splitting the mass becomes necessary~\cite{wang2013linear}. Moreover, there might be multiple solutions, and variational solvers cannot recover any of them. An example is transporting $\{\frac{1}{2}\delta[\xL{1}=(0,-1)], \frac{1}{2}\delta[\xL{2}=(0,1)]\}$ to $\{\frac{1}{2}\delta[\yL{1}=(1,0)], \frac{1}{2}\delta[\yL{2}=(2,0)]\}$. There exist two one-to-one maps but VOT cannot recover either because the target measures cannot be distinguished by the piece-wise linear function $\theta_{\bm{\h}}(\x) = \max_{\IdxCentroid} \{\x \yL{\IdxCentroid} + \hL{\IdxCentroid}\}$, in \ref{sec:vot}. Therefore, when dealing with stochastic GD, having sufficient samples to represent the domain is key to stabilize VWBs. Luckily, increasing the empirical samples adds little computational burden if we parallelly update the correspondence for each empirical according to its nearest neighbor. On the other hand, Sinkhorn iteration-based OT methods produce soft correspondences that unavoidably result from the entropic regularization, making them robust for discrete measures. Occasionally, the soft correspondences are even desirable because they make the correspondences differentiable~\cite{cuturi2019differentiable}; Monge correspondences, however, are basically permutations which are not differentiable. In summary, our VWB producing Monge maps is suitable for clustering or partitioning problems that require binary, sparse correspondence while Sinkhorn distance-based barycenters have been tested in numerous applications in machine learning for producing robust interpolations. 

There are several future directions: 1) In the current implementation, we use exhaustive search to find the nearest centroid for each empirical sample, which takes about $80\%$ of our run time. A faster alternative for nearest neighbor search based on the power distance, which is not a Minkowski distance, will significantly reduce the run time of the VWB;  2) Whether VWBs or WBs for unbalanced measures still induce a generalized metric deserves an answer; 3) Whether our discussion still holds for $ 1 \leq p < 2$ and $p > 2$ deserves an answer; 4) Another branch of computing Monge OT is the multi-scale approach, e.g., \cite{merigot2011multiscale, schmitzer2016sparse, gerber2017multiscale}. It also partitions the target domain into sub-domains. Computing barycenters with multi-scale OT for clustering purposes is worth exploring.


\bibliography{ref}
\bibliographystyle{icml2020}

\end{document}


\twocolumn[
\icmltitle{\ignore{Geometric Clustering via Variational Wasserstein Barycenters\\ . \\}
Variational Wasserstein Barycenters for Geometric Clustering}



\icmlsetsymbol{equal}{*}




\icmlkeywords{Machine Learning, ICML, Optimal Transport, Wasserstein Barycenter, Clustering, K-Means, Variational Method, Co-Clustering}

\vskip 0.3in
]




\appendix

In this appendix, we further illustrate, detail, discuss, and prove several arguments raised in the main paper. 

\section{Proofs for Proposition 1 in 3.1} \label{sec:proof_remark1}

A set of distinct points and measures $\{\bm{\y}, \bm{\NU}\} = \{\yL{\ell}, \NUL{\ell}\}_{\ell = 1}^{K}$

$\theta_{\bm{\h}}(\x) = \underset{\ell}{\max} \{\x \yL{\ell} + \hL{\ell}\}$

$\nabla\theta_{\bm{\h}}(x) = \yL{\ell}$

$\bm{\Region} = \bigcup\limits_{\ell=1}^{K} \RegionL{\ell}$

$\RegionL{\ell} = \{\x \in \X\ |\ \x\yL{\ell} + \hL{\ell} \geq \x\yL{k} + \hL{k}, \forall k \neq \ell\}$

$\bm{\RegionU{'}} = \bigcup\limits_{\ell=1}^{K} \RegionLU{\ell}{'}$

$\RegionL{\ell}' = \{x\in \X\ |\ \|\x - \yL{\ell}\|_2^2 + \varphi_{\ell} \leq \|\x - \yL{k}\|_2^2 + \varphi_{k},\ \forall k \neq \ell \}$

According to Alexandrov~\cite{alexandrov2005convex}, there exists a unique vector $\bm{\h}$ up to a translation, so that $\int_{\RegionL{\ell}} d\MU(x) = \NUL{\ell},\ \forall \ell \in {1, ..., K}$.

Brenier proved in ~\cite{brenier1991polar} that the gradient map $\nabla\theta_{\bm{\h}}(x) \rightarrow \yL{\ell}$ induced by the unique $\bm{\h}$ also minimizes the transportation cost $\int_{\X} \|\x, \nabla\theta_{\bm{\h}}(x)\|_2^2 d\MU(\x)$.

\begin{proposition} \label{th:energy_connection}
1. $\RegionL{\ell} \equiv \RegionLU{\ell}{'}$. 

2. Minimizing $\EnergyL{2}[\bm{\h}]$ is equivalent to maximizing $\EnergyL{4}[\bm{\h}]$. 

3. The minimum point of $\EnergyL{2}[\bm{\h}]$ also minimizes $\EnergyL{3}[\bm{\h}]$.

4.$\bm{\Region}$ in $\EnergyL{2}[\bm{\h}]$ induces the Monge map $\T: x \rightarrow \yL{\ell}$.
\end{proposition}

\subsection{$\RegionL{\ell} \equiv \RegionLU{\ell}{'}$}

\begin{proof}
\begin{equation}
\begin{split}
    \min_{\T} \EnergyL{1}[\T] = & \int_{\X} \|\x - T(\x)\|_2^2 d\MU(\x) \\
    = & \sum_{\ell=1}^{K} \int_{\RegionL{\ell}'} \|\x - T(\x)\|_2^2 d\MU(\x), \\
    s.t. \int_{\RegionL{\ell}'} d\MU(x) &= \NUL{\ell}, \quad \RegionL{\ell}' \eqdef \{\x \in \X\ |\ \T(x) = \yL{\ell}\}. \nonumber
\end{split}
\end{equation}

\begin{equation}
\begin{split}
    & \max_{\bm{\varphi}} \min_{\T} \EnergyL{4}[\bm{\varphi}, \T] \eqdef \\
    & \sum_{\ell=1}^{K} \int_{\RegionL{\ell}'} \|\x - T(\x)\|_2^2 d\MU(\x) + \sum_{\ell=1}^{K} \varphi_{\ell} \left(\int_{\RegionL{\ell}'} d\MU(x) - \NUL{\ell} \right) \\
    = & \sum_{\ell=1}^{K} \int_{\RegionL{\ell}'} \left( \|\x - T(\x)\|_2^2 + \varphi_{\ell} \right) d\MU(\x) - \sum_{\ell=1}^{K} \varphi_{\ell} \NUL{\ell}\\
    = & \sum_{\ell=1}^{K} \int_{\RegionL{\ell}'} \left( \|\x - \yL{\ell}\|_2^2 + \varphi_{\ell} \right) d\MU(\x) - \sum_{\ell=1}^{K} \varphi_{\ell} \NUL{\ell}. \nonumber
\end{split}
\end{equation}

Note that $\bm{\Region}$ or $\bm{\Region}'$ is induced by $\T$, but we omit $\T$ in the notation for simplicity.

Suppose $\TU{*}$ is the minimizer of $\EnergyL{4}[\T]$, then $\TU{*}$ induces the graph $\RegionL{\ell}' = \{x\in \X\ |\ \|\x - \yL{\ell}\|_2^2 + \varphi_{\ell} \leq \|\x - \yL{k}\|_2^2 + \varphi_{k},\ \forall k \neq \ell \}$ because otherwise there would exist $x \in \RegionL{\ell}$ such that $\|\x - \yL{\ell}\|_2^2 + \varphi_{\ell} > \|\x - \yL{k}\|_2^2 + \varphi_{k}$ and that is contradictory to the fact that $\TU{*}$ is the minimizer. Therefore, $\RegionL{\ell} \equiv \RegionLU{\ell}{'}\ \forall\ \ell$.
\end{proof}

\subsection{Minimizing $\EnergyL{2}[\bm{\h}]$ is equivalent to maximizing $\EnergyL{4}[\bm{\h}]$}

\begin{equation}\label{eq:theta}
    \min\ \EnergyL{3}[\bm{\h}] \eqdef \int_{\X} \theta_{\bm{\h}}(\x) d\MU(\x) - \sum_{\ell = 1}^{K}\NU(\yL{\ell})\hL{\ell}. \nonumber
\end{equation}


$\RegionLU{\ell}{'} = \{\x \in \X\ |\ \|\x - \yL{\ell}\|_{2}^{2} + \LagPhiL{\ell} \leq \|\x - \yL{k}\|_{2}^{2} + \LagPhiL{k}, \forall k \neq \ell\}$

\begin{proof}
Given that $\RegionL{\ell} = \RegionLU{\ell}{'}$,
\begin{equation}
    \begin{split}
        \RegionL{\ell}  = & \{\x \in \X\ |\ \x\yL{\ell} + \hL{\ell} \geq \x\yL{k} + \hL{k}, \forall k \neq \ell\} \\ \nonumber
    \end{split}
\end{equation}

\begin{equation}
    \begin{split}
        \RegionLU{\ell}{'} = & \{\x \in \X\ |\ \|\x - \yL{\ell}\|_{2}^{2} + \LagPhiL{\ell} \leq \|\x - \yL{k}\|_{2}^{2} + \LagPhiL{k}, \forall k \neq \ell\} \\
        = & \{\x \in \X\ |\ 2\x\yL{\ell} - \yLU{\ell}{2} - \LagPhiL{\ell} \geq 2\x\yL{k} - \yLU{k}{2} - \LagPhiL{k}, \forall k \neq \ell\} \nonumber
    \end{split}
\end{equation}

\begin{equation}
    \begin{split}
        \Longrightarrow 2\hL{\ell} &= -\yLU{\ell}{2} - \LagPhiL{\ell} \\
          \hL{\ell} &= - \frac{1}{2} \left(\yLU{\ell}{2} + \LagPhiL{\ell}\right) \\
          \LagPhiL{\ell} &= -2\hL{\ell} - \yLU{\ell}{2} \nonumber
    \end{split}
\end{equation}

\begin{equation}
    \begin{split}
        \EnergyL{4}[\bm{\h}] = & \sum_{\ell = 1}^{K} \int_{\RegionL{\ell}} \big(\|\x - \yL{\ell}\|_{2}^{2} + \LagPhiL{\ell} \big) d\MU(\x) + \sum_{\ell = 1}^{K} \LagPhiL{\ell}\NU(\yL{\ell}) \\ 
        = & \sum_{\ell = 1}^{K} \int_{\RegionL{\ell}} \xU{2} d\MU(\x) - 2 \sum_{\ell=1}^{K} \int_{\RegionL{\ell}} \x\yL{\ell} d\MU(\x) \\
        & + \sum_{\ell=1}^{K} \int_{\RegionL{\ell}} \yLU{\ell}{2} d\MU(\x) \\
        & - 2\sum_{\ell=1}^{K} \int_{\RegionL{\ell}} \hLU{\ell}{2} d\MU(\x) - \sum_{\ell=1}^{K} \int_{\RegionL{\ell}} \yLU{\ell}{2} d\MU(\x) \\
        & + 2\sum_{\ell = 1}^{K} \hL{\ell}\NU(\yL{\ell}) + \sum_{\ell = 1}^{K} \yLU{\ell}{2}\NU(\yL{\ell}) \\
 \nonumber
    \end{split}
\end{equation}
\begin{equation}
    \begin{split}
        = & - 2 \sum_{\ell=1}^{K} \int_{\RegionL{\ell}} \x\yL{\ell} d\MU(\x) - 2\sum_{\ell=1}^{K} \int_{\RegionL{\ell}} \hL{\ell} d\MU(\x) \\
        & + 2\sum_{\ell = 1}^{K} \hL{\ell}\NU(\yL{\ell}) + constants \\
        = & - 2 \sum_{\ell=1}^{K} \int_{\RegionL{\ell}} \left(\x\yL{\ell} + \hL{\ell}\right) d\MU(\x) \\
        & - 2 \sum_{\ell = 1}^{K} \hL{\ell}\NU(\yL{\ell}) + constants \\
        = & - 2 \EnergyL{2}[\bm{\h}] + constants. \nonumber
    \end{split}
\end{equation}
Therefore, minimizing $\EnergyL{2}[\bm{\h}]$ is equivalent to maximizing $\EnergyL{4}[\bm{\h}]$. 
\end{proof}

\subsection{The minimum point of $\EnergyL{2}[\bm{\h}]$ also minimizes $\EnergyL{3}[\bm{\h}]$}
\begin{proof}
\begin{equation}
    \begin{split}
        \nabla \EnergyL{3}[\bm{\h}] =& \bigg\{ \frac{\partial \EnergyL{3}[\bm{\h}]}{\partial \hL{\ell}} \bigg\}_{\ell=1}^{K}\\
        =& \bigg\{\int_{\X} \frac{\partial \theta_{\bm{\h}}}{\hL{\ell}} d\MU(x) - \NUL{\ell} \bigg\}_{\ell=1}^{K}\\
        =& \bigg\{\int_{\RegionL{\ell}} d\MU(x) - \NUL{\ell} \bigg\}_{\ell=1}^{K}\\ \nonumber
    \end{split}
\end{equation}
\begin{equation}
    \begin{split}
        \nabla \EnergyL{2}[\bm{\h}] =& \bigg\{ \frac{\partial \EnergyL{2}[\bm{\h}]}{\partial \hL{\ell}} \bigg\}_{\ell=1}^{K}\\
        =& \bigg\{\int_{\RegionL{\ell}} d\MU(x) - \NUL{\ell} \bigg\}_{\ell=1}^{K}\\ \nonumber
    \end{split}
\end{equation}

Therefore, $\nabla \EnergyL{2}[\bm{\h}] = \nabla \EnergyL{3}[\bm{\h}]$. As per proved in~\cite{gu2013variational}, both $\EnergyL{2}[\bm{\h}]$ and $\EnergyL{3}[\bm{\h}]$ are strictly convex which means their minimum points are $\nabla \EnergyL{2}[\bm{\h}] \rightarrow \bm{0}$ or $\nabla \EnergyL{3}[\bm{\h}] \rightarrow \bm{0}$. In order to solve $\EnergyL{3}[\bm{\h}]$, we can then instead solve $\EnergyL{2}[\bm{\h}]$ for the optimal $\bm{\hU{*}}$.
\end{proof}

\subsection{$\bm{\Region}$ in $\EnergyL{2}[\bm{\h}]$ induces the Monge map $\T: x \rightarrow \yL{\ell}$}

\begin{proof}

\begin{equation}\label{eq:vot}
    \begin{gathered}
    \min\ \EnergyL{2}[\bm{\h}] \eqdef \int_{\bm{0}}^{\bm{\h}} \sum_{\ell=1}^{K} \int_{\RegionL{\ell}} d\MU(\x) d\hL{\ell} - \sum_{\ell = 1}^{K}\NU(\yL{\ell})\hL{\ell}, \\
    \RegionL{\ell} = \{\x \in \X\ |\ \x\yL{\ell} + \hL{\ell} \geq \x\yL{k} + \hL{k}, \forall k \neq \ell\}. \nonumber
    \end{gathered}
\end{equation}

This is indeed true since $\bm{\Region} = \bm{\Region}'$ and $\bm{\Region}'$ induces the Monge map.
\end{proof}

\section{Derivatives of VWBs} \label{sec:vwb_gd}
\begin{equation}\label{eq:vwb}
\begin{gathered}
    \min_{\{\bm{\hL{i}}\}_{i=1}^{N}} \EnergyL{5}[\NU] \\
    \eqdef \frac{1}{N} \sum_{i=1}^{N} \left( \int_{\bm{0}}^{\bm{\hL{i}}} \sum_{\ell=1}^{K} \int_{\RegionL{i, \ell}} d\MUL{i}(x) d\hL{i, \ell} - \sum_{\ell = 1}^{K}\NU(\yL{\ell})\hL{i, \ell} \right)\\ \nonumber
\end{gathered}
\end{equation}

The gradient of $\EnergyL{5}[\bm{\h}]$ is 
\begin{equation}
\begin{split}
    \nabla \EnergyL{5}[\bm{\h}] &= \left\{\left\{\pdv{\EnergyL{5}}{\hL{i, \ell}} \right\}_{\ell = 1}^{K}\right\}_{i=1}^{N} \\
    &= \left\{\left\{\int_{\RegionL{i, \ell}} d\MUL{i}(x) - \NU(\yL{\ell})\right\}_{\ell = 1}^{K}\right\}_{i=1}^{N}. \nonumber
\end{split}
\end{equation}

The Hessian of $\EnergyL{5}[\bm{\h}]$ is then 

\begin{equation}
\begin{split}
    H &= \left(\frac{\partial^{2} \EnergyL{5}[\bm{\h}] }{\partial \hL{i, \ell}\partial \hL{j, k}} \right) \\
	&=
	\left\{
	\begin{array}{ll}
		 \sum_{k} \cfrac{\int_{f_{i,\ell, k}}\mu(x)dx}{\|\yL{\ell}-\yL{k}\|}, &\quad i = j,\ \forall k, s.t.\  f_{i, \ell, k} \neq \emptyset, \\[3ex]
		 -\cfrac{\int_{f_{i, \ell, k}}\MUL{i}(x)dx}{\|\yL{\ell}-\yL{k}\|}, &\quad i = j,\ f_{i, \ell, k} \neq \emptyset, \\[3ex]
		 0, &\quad i \neq j.
	\end{array}
	\right. \nonumber
	\label{eq:hessian}
\end{split}
\end{equation}
$f_{i,\ell, k} = \RegionL{i, \ell} \cap \RegionL{i, k}$.

This is a very sparse matrix since variables from different VOT problem, $i, j$, are excluded from each other and in each VOT problem one power Voronoi cell is only adjacent to a few other cells. 

\section{Updating Both Supports $\bm{\y}$ and Measures $\bm{\NU}$ in 4.1} \label{sec:wbs_update_ynu}

We show a comparison between our method and \cite{ye2017fast} on fitting a Gaussian mixture to the target domain. The two methods lead to similar decision boundaries but our embedding is more evenly distributed into the target domain according to the density.

\begin{figure}[t]
\begin{center}
\centerline{\includegraphics[width=\linewidth]{plots.pdf}}
\caption{Matching two Gaussian mixtures with \cite{ye2017fast} and our method. Updating both supports and measures may result in centroids not evenly distributed into the target domain, which although may not affect the classification boundary in this example.}
\end{center}
\end{figure}

\section{Critical Point of VWBs w.r.t. $\{\NUL{\ell}\}_{\ell=1}^{K}$ in 4.1} \label{sec:nu}

We update $\NUL{\ell}$ so that $\RegionL{i}\ \forall\ i$ induced by the optimal $\TLU{i}{*}$ forms an unweighted Voronoi diagram where each empirical sample $\x$ is mapped to its nearest $\yL{\ell}$ w.r.t. the quadratic Euclidean distance. Now, by contradiction, suppose we update $\NUL{\ell}$ in such as way that the OT map does not form an unweighted Voronoi diagram, i.e. $\TL{i}' \neq \TLU{i}{*}$. Then, there must be some sample $\x$ that is mapped to some $\yL{k} \neq \yL{\ell}$ which is not the nearest centroid. In this case, the total transportation cost increases because $\|\x - \yL{k}\|_2^2 > \|\x - \yL{\ell}\|_2^2$. Therefore, the minimum total cost w.r.t. $\NUL{\ell}$ is achieved when $\NUL{\ell}$ induces the OT map that forms an unweighted Voronoi diagram. When we update $\NUL{\ell}$, we only need to construct an unweighted Voronoi diagram according to the current supports $\bm{\y}$ and assign the total mass within each cell to its corresponding support. This Voronoi diagram itself will also be the OT map like Lloyd's algorithm.

\section{Algorithm of Computing VWBs in 4.1} \label{sec:algorithm}

\begin{equation}\label{eq:vwb_dh}
\begin{gathered}
    \nabla \EnergyL{5}[\bm{\hL{i}}] = \left\{\pdv{\EnergyL{5}}{\hL{i, \ell}} = \int_{\RegionL{i, \ell}} d\MUL{i}(x) - \NUL{\ell}\right\}_{\ell = 1}^{K},\\
    \bm{\hL{i}}^{(t+1)} = \bm{\hL{i}}^{(t)} - \eta \nabla \EnergyL{5}[\bm{\hL{i}}].
\end{gathered}
\end{equation}

\begin{equation} \label{eq:vwb_dy}
\begin{split}
    \yLU{\ell}{*}
    = \frac{\sum_{i = 1}^{N} \int_{\RegionL{i, \ell}} xd\MUL{i}(\x)}{N\sum_{i = 1}^{N} \int_{\RegionL{i, \ell}}d\MUL{i}(\x)} 
    \approx \frac{\sum_{i = 1}^{N} \sum_{\x \in \RegionL{i, \ell}} x \MUL{i}(\x)}{N\sum_{i = 1}^{N} \sum_{\x \in \RegionL{i, \ell}} \MUL{i}(\x)},
\end{split}
\end{equation}

\begin{equation} \label{eq:vwb_dv}
\begin{split}
    \NUUL{*}{\ell}
    = \frac{1}{N}\sum_{i = 1}^{N} \int_{\RegionLU{i, \ell}{*}} d\MUL{i}(\x)
    \approx \frac{1}{N} \sum_{i = 1}^{N} \sum_{ \x \in \RegionLU{i, \ell}{*}} \MUL{i}(\x),
\end{split}
\end{equation}

\begin{algorithm}[ht]
    \caption{Variational Wasserstein Barycenters}
    \label{alg:vwb}
\begin{algorithmic}
    \STATE {\bfseries Input:}  $\{\MUL{i}\}_{i=1}^{N}$, $K$
    \STATE Initialize $\NU \in \MUS(\Y)$.
    \REPEAT
        \STATE Compute $\TL{i}$ between $\NU$ and each $\MUL{i}$ by solving \eqref{eq:vwb}.
        \STATE Update partition or assignment (if discrete), $\bm{\Region}$, according to $\bm{\T}$.
        \IF{Free support} 
            \STATE {Update $\yL{\ell}\ \forall\ \ell$ according to~\eqref{eq:vwb_dy}}
        \ENDIF
        \IF {Free measure}
            \STATE {Update $\NUL{\ell}\  \forall\ \ell$ according to~\eqref{eq:vwb_dv}}
        \ENDIF
    \UNTIL{$\NU$ converges}
\end{algorithmic}
\end{algorithm}

\section{Proofs of Propositions in 4.2} \label{sec:vwbnmetric}

From now, we use a few new notations here that are more clear. We denote a collection of probability distributions by $\MUNL{1:N} \eqdef \{\MUL{i}\}_{i=1}^{N}$ and a certain permutation by $\bm{\MU}_{\sigma(1:N)}$.
Then, we change the notation for the total Wasserstein distance induced by the Wasserstein barycenter as $\WBL{\NU}(\MUNL{1:N}) \longleftarrow \WBL{1:N}(\NU)$ because the variable is the barycenter $\NU$ and the input is the marginals $\MUNL{1:N}$.

\begin{equation} \label{eq:wb_nmetric}
    \WBL{\NU}(\MUNL{1:N}) \eqdef \underset{\NU \in \MUS(\Y)}{\inf} \frac{1}{N}\sum_{i = 1}^{N} \WUL{2}{2}(\MUL{i}, \NU),
\end{equation}
\ignore{
According to~\cite{kiss2018generalization}, for a metric space $(\MUS(X), d)$ and $N \geq 2$,
$$d(\MUL{1}, ..., \MUL{N}) = \underset{\MUL{N+1 \in \MUS(X)}}{\inf} \sum_{i=1}^{N}d(\MUL{i}, \MUL{N+1}) $$
defines a generalized metric. If we embody the metric with the quadratic Wasserstein distance $\WLU{2}{2}$, then $\WBL{\NU}(\MUNL{1:N})$ is indeed an n-metric. For completeness, we prove the n-metric properties as follows:
}

\subsection{Proposition 2}
\begin{proposition}
$\WBL{\NU}(\MUNL{1:N})$ defines a generalized metric among $\MUNL{1:N}$, $\forall\ N \geq 2$. Specifically, $\WBL{\NU}(\MUNL{1:N})$ satisfies the following properties. \\
1) Non-negativity: $\WBL{\NU}(\MUNL{1:N}) \geq 0$.\\
2) Symmetry: $\WBL{\NU}(\MUNPL{1}) = \WBL{\NU}(\MUNPL{2})$, where $\sigma_1(1:N)$ and $\sigma_2(1:N)$ are different permutations of the set ${1:N}$.\\
3) Identity: $\WBL{\NU}(\MUNL{1:N}) = 0 \Longleftrightarrow \MUL{i} = \MUL{j}, \forall i \neq j$.\\
4) Triangle inequality: $\WBL{\NU}(\MUNL{1:N}) \leq \sum_{i=1}^{N} \WBL{\NU}(\MUNL{1:N+1\backslash i})$.
\end{proposition}

\begin{proof}
Suppose the Wasserstein barycenter of $\MUNL{i:N}$ is $\NULU{1:N}{*}$, i.e. $\WBL{\NU}(\MUNL{1:N}) = \frac{1}{N}\sum_{i = 1}^{N} \WUL{2}{2}(\MUL{i}, \NUU{*})$.

1) Non-negativity: $\WBL{\NU}(\MUNL{1:N}) \geq 0$.

Since $\WLU{2}{2}(\MUL{i}, \NUU{*}) \geq 0$, $\WBL{\NU}(\MUNL{1:N})$ is obviously not negative. The equal sign holds when $\WLU{2}{2}(\MUL{i}, \NUU{*})$ is zero for all $i$. When that happens, $\MUL{i} = \NUU{*}$ for all $i$. It also implies that all marginals are equal to each other, i.e. $\MUL{i} = \MUL{j},\ \forall\ i \neq j$.

2) Symmetry: $\WBL{\NU}(\MUNPL{1}) = \WBL{\NU}(\MUNPL{2})$

This is true since the summation of the WDs is symmetric.

3) Identity: $\WBL{\NU}(\MUNL{1:N}) = 0 \Longleftrightarrow \MUL{i} = \MUL{j}, \forall i \neq j$.

This is true according to our discussion in 1).

4) Triangle inequality: $\WBL{\NU}(\MUNL{1:N}) \leq \sum_{i=1}^{N} \WBL{\NU}(\MUNL{1:N+1\backslash i})$.

\begin{equation}
    \WBL{\NU}(\MUNL{1:N+1\backslash i}) = \frac{1}{N}\left( \sum_{j=1}^{N+1} \WUL{2}{2}(\MUL{j}, \NUU{*}) - \WUL{2}{2}(\MUL{i}, \NUU{*})\right), \nonumber
\end{equation}
\begin{equation}
    \sum_{i=1}^{N} \WBL{1:N+1 \backslash i}(\NU) = \sum_{j=1}^{N+1} \WUL{2}{2}(\MUL{j}, \NUU{*}) - \frac{1}{N} \sum_{i=1}^{N} \WUL{2}{2}(\MUL{i}, \NUU{*}).
    \nonumber
\end{equation}
To prove 4), we need to show that 
\begin{equation}
  \frac{1}{N} \sum_{i=1}^{N} \WUL{2}{2}(\MUL{i}, \NUU{*}) \leq \sum_{j=1}^{N+1} \WUL{2}{2}(\MUL{j}, \NUU{*}) - \frac{1}{N} \sum_{i=1}^{N} \WUL{2}{2}(\MUL{i}, \NUU{*}) \nonumber
\end{equation}
\begin{equation}
  \frac{2}{N} \sum_{i=1}^{N} \WUL{2}{2}(\MUL{i}, \NUU{*}) \leq \sum_{j=1}^{N} \WUL{2}{2}(\MUL{j}, \NUU{*}) + \WUL{2}{2}(\MUL{N+1}, \NUU{*}).\nonumber
\end{equation}
This is indeed true for $N \geq 2$ because the left and the first term on the right cancel out and the second term on the right is non-negative. The equal sign holds if $N=2$ and $\MUL{N+1}=\NUU{*}$, which means the additional $\MUL{N+1}$ \textit{is} the barycenter.
\end{proof}

\subsection{Proposition 3}

\begin{proposition}\label{the:wb_nmetric_bigon}
The bound of the triangle inequality in Proposition~2 can be tightened by a linear factor. Specifically, we have $\frac{N}{2} \WBL{\NU}(\MUNL{1:N}) \leq \sum_{i=1}^{N} \WBL{\NU}(\MUNL{1:N+1})$.
\end{proposition}

\begin{proof}
By adding a linear factor $\frac{N}{2}$ in the front of the l.h.s, we cancel out the $\frac{2}{N}$ in the last line of the proof above.
\begin{equation}
  \sum_{i=1}^{N} \WUL{2}{2}(\MUL{i}, \NUU{*}) \leq \sum_{j=1}^{N} \WUL{2}{2}(\MUL{j}, \NUU{*}) + \WUL{2}{2}(\MUL{N+1}, \NUU{*}).\nonumber
\end{equation}
The inequality still holds. The equal sign holds if $\MUL{N+1}=\NUU{*}$.
\end{proof}

\subsection{Corollary 2} \label{sec:vwb2metric}
\begin{corollary}\label{the:vwb_2metric}
$\VWBL{1:2}$ induces a (2-)metric between $\MUL{1}$ and $\MUL{2}$. In addition, $\VWBL{1:2}$ is lower-bounded by $\frac{1}{4}\WLU{2}{2}(\MUL{1},\MUL{2})$ when \ignore{$\MUL{1}$, $\MUL{2}$, and $\NU$ have the same number of supports}$|\MUL{i}| = |\NU| = K$.
\end{corollary}

\begin{proof}
First, since Proposition~ 2 and~ 3 do not have any restriction on the continuity of the barycenter, the barycenters produced by our method, VWBs, inherit the metric properties. When $N=2$, the n-metric regresses to a 2-metric. The issue is enforcing the equal signs. For all the marginals to equal to the discrete barycenter, they must have the same number of supports, i.e. $|\MUL{i}| = |\NU| = K$.

When $|\MUL{i}| = |\NU| = K$, we have 
$$\WL{2}(\MUL{1}, \MUL{2}) = \WL{2}(\MUL{1}, \NU{*}) + \WL{2}(\MUL{2}, \NU{*})$$
and 
$$\WL{2}(\MUL{1}, \NU{*}) = \WL{2}(\MUL{2}, \NU{*}) = \frac{1}{2} \WL{2}(\MUL{1}, \MUL{2}).$$
Thus, 
\begin{equation}
\begin{split}
\VWBL{\NU}(\MUN) =& \frac{1}{2} \WLU{2}{2}(\MUL{1}, \NU{*}) + \frac{1}{2} \WLU{2}{2}(\MUL{2}, \NU{*}) \\
=& \frac{1}{2} \cdot \frac{1}{4} \WLU{2}{2}(\MUL{1}, \MUL{2}) + \frac{1}{2} \cdot \frac{1}{4} \WLU{2}{2}(\MUL{1}, \MUL{2})\\
=& \frac{1}{4} \WLU{2}{2}(\MUL{1}, \MUL{2}).
\nonumber
\end{split}
\end{equation}
\end{proof}

\section{Visualizing a 3-Metric for 4.2} \label{sec:vis_nmetric}
\begin{figure*}[t]
\begin{center}
\centerline{\includegraphics[width=0.8\textwidth]{triangle_new.pdf}}
\caption{Triangle inequality of the 3-metric induced by three marginals and their Wasserstein barycenter.}
\label{fig:triangle}
\end{center}
\end{figure*}
We visualize the triangle inequality of a 3-metric in Figure~\ref{fig:triangle}. Suppose we have the setup in 3.1 and four probability measures $\MUNL{1:4}$. The Wasserstein barycenter of the marginals, $\MUL{1}$, $\MUL{2}$, and $\MUL{3}$, is defined as
$$    
\WBL{\NU}(\MUNL{1:3}) = \underset{\NU}{\inf}\ \frac{1}{3} \sum_{i=1}^{3} \WUL{2}{2}(\MUL{i}, \NU)
$$
Then, according to the triangle inequality, for any $\MUL{4}$,
$$\WBL{\NU}(\MUNL{1,2,3}) \leq \WBL{\NU}(\MUNL{1,2,4}) + \WBL{\NU}(\MUNL{1,3,4}) + \WBL{\NU}(\MUNL{2,3,4})$$

\begin{figure*}[thb]
\begin{center}
\centerline{\includegraphics[width=0.7\textwidth]{icp_color.pdf}}
\caption{Interpolating a ``mean'' shape from two ``anchors'' by regularizing the OT maps with rigid transformation to preserve the global structure.}
\label{fig:icp}
\end{center}
\end{figure*}

\section{Details of Point Cloud Interpolation in 6.1}
In this experiment, we are given two Kittens off by an unknown rigid transformation. Our goal is to interpolate, by computing a regularized Wasserstein barycenter, a new Kitten in between that is rigid to the original Kittens and the amount of translation and rotation is linear to the weights of the two original Kittens.

The marginal Kittens each have $7,805$ sample points. We assume all the samples have equal weights. They are apart from each other by a rigid transformation composed by a random translation vector $t$ and a random rotation matrix $r$. In this example, they are as follows:
\begin{equation*}
  t = \begin{bmatrix} -1.97\\ -0.73\\ -0.30\end{bmatrix}
  \quad
  r = \begin{bmatrix} 0.87 &-0.23 & 0.44 \\ 0.41 & 0.84 & -0.36 \\ -0.30 & 0.49 & 0.82\end{bmatrix}
\end{equation*}

The barycenter Kitten w.r.t. the VWD (variational Wasserstein distance) has $780$ supporting Dirac measures. The regularization strength, $\lambda$, is $10$. One of the post-processing options to transport all the samples from the marginals is that for each sample find its nearest 3 or more cluster centers and use inverse barycenter coordinates to find its new location on the target Kitten in the middle.

\section{Color Histograms of Images in 6.2} \label{sec:color}
In Figure~\ref{fig:histogram}, we show the color histograms of the images in 6.2. The code file includes all the original images for reference. In each histogram, we plot the pixels at the location according to their original colors, but paint them according to the new quantized color. The figure shows that the centroids are distributed differently in K-means and VOT. Each cluster in VOT is guaranteed to enclose the same number of pixels as others. In the last row, all the pixels are quantized to the same shared centroids across three images. Original images and vector graphics are included in the package.

\begin{figure}[t]
\begin{center}
\centerline{\includegraphics[width=0.94\linewidth]{color_all.pdf}}
\caption{Images and their histograms in the experiments in 6.2. Each histogram besides those of the original images have black dots representing the quantization centroids from different methods. Our method enforces an equal number of samples for each centroid, leading to balanced quantization. Original images are included in the submission package for zoom-in.}
\label{fig:histogram}
\end{center}
\end{figure}

\bibliography{ref}
\bibliographystyle{icml2020}